\documentclass[10pt,twocolumn,twoside]{IEEEtran}
\usepackage[center]{caption}
\usepackage{cite}
\usepackage{amsmath,amssymb,amsfonts}
\usepackage{graphicx}
\usepackage{textcomp}
\usepackage{graphicx,float,stfloats,subfig,epstopdf}
\usepackage{amsmath,bbm,bm}
\usepackage{amsfonts}
\usepackage{amssymb,color}
\usepackage{algorithm,algpseudocode,comment}      % algorithm
\usepackage{url,graphicx}
\newtheorem{assumption}{Assumption}
\newtheorem{theorem}{Theorem}
\newtheorem{lemma}{Lemma}
\newtheorem{remark}{Remark}

\newcommand{\eps}{\epsilon}

\newtheorem{definition}{Definition}
\newcommand{\bE}{\mathbb{E}}

\newcommand{\cS}{\mathcal{S}}

\newcommand{\bN}{\mathbb{N}}
\newcommand{\cC}{\mathcal{C}}
\newcommand{\ust}{^{\star}}
\newcommand{\bR}{\mathbb{R}}
\newcommand{\te}{\theta}

\newcommand{\cI}{\mathcal{I}}

\newcommand{\bP}{\mathbb{P}}
\newcommand{\cE}{\mathcal{E}}
\newcommand{\cG}{\mathcal{G}}
\newcommand{\cB}{\mathcal{B}}
\newcommand{\cA}{\mathcal{A}}

\newcommand{\cL}{\mathcal{L}}

\newcommand{\ur}{^{(R)}}

\newcommand{\uci}{^{(i)}}
\newcommand{\id}{\mathbbm{1}}
\newcommand{\up}{^{\prime}}
\newcommand{\cF}{\mathcal{F}}

\newcommand{\nal}[1]{\begin{align*}#1\end{align*}}
\newcommand{\ali}[1]{\begin{align}#1\end{align}}
\newcommand{\alig}[1]{\begin{align}#1\end{align}}

\title{Learning in Markov Decision Processes under Constraints}

\begin{document}
	\author{Rahul Singh, \IEEEmembership{Member, IEEE}, Abhishek Gupta, and Ness Shroff , \IEEEmembership{Fellow, IEEE}
	%	\thanks{This paragraph of the first footnote will contain the date on 
	%		which you submitted your paper for review. It will also contain support 
	%		information, including sponsor and financial support acknowledgment. For 
	%		example, ``This work was supported in part by the U.S. Department of 
	%		Commerce under Grant BS123456.'' }
		\thanks{Rahul Singh  is with the
			Department of ECE, Indian Institute of Science, \\
			Bengaluru, Karnataka, India. email: rahulsingh@iisc.ac.in.}
		\thanks{           Abhishek Gupta and Ness Shroff are with the
			Department of ECE, Ohio State University,
			Columbus, OH, USA. email: gupta.706@osu.edu,shroff@ece.osu.edu}}

	\maketitle

	\date{Received: date / Accepted: date}
	
\maketitle
\begin{abstract}
We consider reinforcement learning (RL) in Markov Decision Processes in which an agent repeatedly interacts with an environment that is modeled by a controlled Markov process. At each time step $t$, it earns a reward, and also incurs a cost-vector consisting of $M$ costs. We design model-based RL algorithms that maximize the cumulative reward earned over a time horizon of $T$ time-steps, while simultaneously ensuring that the average values of the $M$ cost expenditures are bounded by agent-specified thresholds $c^{ub}_i,i=1,2,\ldots,M$. 

In order to measure the performance of a reinforcement learning algorithm that satisfies the average cost constraints, we define an $M+1$ dimensional regret vector that is composed of its reward regret, and $M$ cost regrets. The reward regret measures the sub-optimality in the cumulative reward, while the $i$-th component of the cost regret vector is the difference between its $i$-th cumulative cost expense and the expected cost expenditures $Tc^{ub}_i$. 

We prove that the expected value of the regret vector of UCRL-CMDP, is upper-bounded as $\tilde{O}\left(T^{2\slash 3}\right)$, where $T$ is the time horizon. We further show how to reduce the regret of a desired subset of the $M$ costs, at the expense of increasing the regrets of rewards and the remaining costs. To the best of our knowledge, ours is the only work that considers non-episodic RL under average cost constraints, and derive algorithms that can~\emph{tune the regret vector} according to the agent's requirements on its cost regrets.
\end{abstract}
\section{Introduction}\label{sec:intro}
Reinforcement Learning (RL)~\cite{DBLP:books/lib/SuttonB98} involves an agent repeatedly interacting with an environment modelled by a Markov Decision Process (MDP)~\cite{puterman2014markov}. More specifically, consider a controlled Markov process~\cite{puterman2014markov} $s_t,~t=1,2,\ldots,T$. At each discrete time $t$, an agent applies control $a_t$. State-space, and action space are denoted by $\cS$ and $\cA$ respectively, and are assumed to be finite. 
The controlled transition probabilities are denoted by $p:= \{ p(s,a,s^{\prime}) : s,s^{\prime}\in \cS, a\in\cA\}$. Thus, $p(s,a,s\up)$ is the probability that the system state transitions to state $s\up$ upon applying action $a$ in state $s$. The probabilities $p(s,a,s\up)$ are not known to the agent. At each discrete time $t=1,2,\ldots,T$, the agent observes the current state of the environment $s_t$, applies control action $a_t$, and earns a reward $r_t$ that is a known function of $(s_t,a_t)$. When the agent applies an action $a$ in the state $s$, then it earns a reward equal to $r(s,a)$ units. The agent does not know the controlled transition probabilities $p(s,a,s\up)$ that describe the system dynamics of the environment. The performance of an agent or a RL algorithm is measured by the cumulative rewards that it earns over the time horizon. 

However in many applications, in addition to earning rewards, the agent also incurs costs at each time. The underlying physical constraints impose constraints on its cumulative cost expenditures, so that the agent needs to balance its reward earnings with the cost accretion while also simultaneously learning the choice of optimal decisions, all in an \emph{online manner}. \par
As a motivating example, consider a single-hop wireless network that consists of a wireless node that transmits data packets to a receiver over an unreliable wireless channel. The channel reliability, i.e., the probability that a transmission at time-step $t$ is successful, depends upon the instantaneous channel state $cs_t$ and the transmission power $a_t$. Thus, for example, this probability is higher when the channel is in a good state, or if transmission is carried out at higher power levels. The transmitter stores packets in a buffer, and its queue length at time $t$ is denoted by $Q_t$. The wireless node is battery-operated, and packet transmission consumes power. Hence, it is desired that the average power consumption is minimal. An appropriate performance metric for networks is the average queue length $\left(\bE \sum_{t=1}^{T} Q_t\right)\slash T$~\cite{sennott2009stochastic}, and hence it is required that the average queue length stays below a certain threshold. The AP has to choose $a_t$ adaptively so as to minimize the power consumption $\left(\bE \sum_{t=1}^{T} a_t\right)\slash T$, or equivalently maximize $\left(\bE \sum_{t=1}^{T} -a_t\right)\slash T$, while simultaneoulsy ensure that the average queue length is below a user-specified threshold, i.e. $\left(\bE \sum_{t=1}^{T} Q_t\right)\slash T\le c^{ub}$. In this example, the state of the ``environment'' at time $t$  is given by the queue length and the channel state $(Q_t,cs_t)$. Thus, it might be ``optimal'' to utilize high transmission power levels only when the instantaneous queue length $Q_t$ is large or the wireless channel's state $cs_t$ is good. Such an adaptive strategy saves energy by transmitting at lower energy levels at other times. Since channel reliabilities are typically not known to the transmitter node, it does not know the transition probabilities $p(s,a,s\up)$ that describe the controlled Markov process $(Q_t,cs_t)$. Hence, it cannot compute the expectations of the average queue lengths and average power consumption for a fixed control policy, and needs to devise appropriate learning policies to optimize its performance under average-cost constraints. RL algorithms that we propose in this work solve exactly these classes of problems. Infact many important network control problems can be solved within the framework of constrained Markov decision processes (CMDPs). For example,~\cite{lazar1983optimal,hsiao1991optimal} utilize CMDPs in order to maximize the throughput in a stochastic network, where the network operator wants to satisfy constraints on delays,~while ~\cite{nain1986optimal} design policies that make dynamic decisions regarding network access in networks shared by different types of traffic. Similarly, it has been used in~\cite{singh2018throughput,singh2021adaptive} in order to maximize the timely throughput\footnote{Throughput derived from those packets which reach their destination within their deadline.} in stochastic networks.~If the network\slash system parameters are unknown, then the CMDP can be posed as a linear program (LP), and solved efficiently. However, in practice, network parameters are seldom known to the network operator, and it needs to design algorithms which ``learn'' the optimal policies in an ``optimal'' manner. Our work addresses precisely this issue.
%while~\cite{feinberg1994optimality} address the issue of admission control and routing in networks shared by mutiple flows in which the goal is to maximize the weighted sum of customers served, while simultaneously satisfying constraints on the blocking probability.

\section{Previous Works and Our Contributions}
%We begin by describing previous related works and how our work differs from these. We also discuss how our work addresses the issues that were overlooked in previous works. Existing related works can be broadly classified into the following four categories: (i) RL without constraints, (ii) RL for constrained MDPs, (iii) Multi-Armed bandits with Constraints, and (iv) Safe RL. We discuss these separately below.
%\vspace{.5cm}

\textit{RL Algorithms for unconstrained MDPs}: RL problems without constraints are well-understood by now. Works such as~\cite{brafman2002r,auer2007logarithmic,bartlett2009regal,jaksch2010near} develop algorithms using the principle of ``optimism under uncertainty.'' UCRL2 of~\cite{jaksch2010near} is a popular RL algorithm that has a regret bound of $\tilde{O}(D(p)S\sqrt{AT})$, where $D(p)$ is the diameter~\cite{jaksch2010near} of the MDP $p$; the algorithms proposed in our work are based on UCRL2.

\vspace{.5cm}

\textit{RL Algorithms for Constrained MDPs}: \cite{altman1991adaptive} is an early work on optimally controlling unknown MDPs under average cost constraints. It utilizes the certainty equivalence (CE) principle, i.e., it applies controls that are optimal under the assumption that the true (but unknown) MDP parameters are equal to the empirical estimates, and also occasionally resorts to ``forced explorations.'' This algorithm yields asymptotically (as $T\to\infty$) the same reward rate as the case when the MDP parameters are known. However, analysis is performed under the assumption that the CMDP is \emph{strictly feasible}. Moreover the algorithm lacks finite-time performance guarantees (bounds on regret). Unlike~\cite{altman1991adaptive}, we do not assume strict feasibility; infact we show that the use of \emph{confidence bounds} allows us to get rid of the strict feasibility assumption.~\cite{borkar2005actor} derives a learning scheme based on multi time-scale stochastic approximation~\cite{borkar1997stochastic}, in which the task of learning an optimal policy for the CMDP is decomposed into that of learning the optimal value of the dual variables, which correspond to the price of violating the average cost constraints, and that of learning the optimal policy for an unconstrained MDP parameterized by the dual variables. However, the proposed scheme lacks finite-time regret analysis, and might suffer from a large regret. Prima facie, this layered decomposition might not be optimal with respect to the sample-complexity of the online RL problem. The works~\cite{achiam2017constrained,liu2019ipo,tessler2018reward,uchibe2007constrained} design policy-search algorithms for constrained RL problems. However unlike our work, they do not utilize the concept of regret vector, and their theoretical guarantees need further research. After the first draft of our work was published online, there appeared a few manuscripts\slash works that address various facets of learning in CMDPs, and these have some similarity with our work. For example~\cite{qiu2020upper} considers episodic RL problems with constraints in which the reward function is time-varying. Similarly,~\cite{efroni2020exploration} also considers episodic RL in which the state is reset at the beginning of each episode. In contrast, we deal exclusively with non-episodic infinite horizon RL problems.~In fact, as we show in our work, the primary difficulty in non-episodic constrained RL arises due to the fact that it is not possible to simultaneously ``control\slash upper-bound'' the reward and $M$ costs during long runs of the controlled Markov process. Consequently, in order to control the regret vector, we make the assumption that the underlying MDP is unichain. However, this problem does not occur in the episodic RL case~\cite{efroni2020exploration,qiu2020upper} since the state is reset. Secondly, unlike the algorithms provided in our work,~\cite{qiu2020upper,efroni2020exploration} do not allow the agent to tune the regret vector.

Our contributions are summarized as follows.
\begin{enumerate}
	\item We initiate the problem of designing RL algorithms that maximize the cumulative rewards while simultaneously satisfying average cost constraints. We propose an algorithm which we call UCRL for CMDPs, henceforth abbreviated as UCRL-CMDP. UCRL-CMDP is a modification of the popular RL algorithm UCRL2 of~\cite{jaksch2010near} that utilizes the principle of optimism in the face of uncertainty (OFU) while making decisions. Since an algorithm that utilizes OFU does not need to satisfy cost constraints (this is briefly discussed in Section~\ref{subsec:failure}), we modify OFU appropriately and derive the principle of~\emph{balanced optimism in the face of uncertainty} (BOFU). Under the BOFU principle, at the beginning of each RL episode, the agent has to solve for (i) an MDP, and (ii) a controller, such that the average costs of a system in which the dynamics are described by (i), and which is controlled using (ii), are less than or equal to the cost constraints. This is summarized in Algorithm~\ref{algo:ucb_cmdp}.
	\item In order to quantify the finite-time performance of an RL algorithm that has to perform under average cost constraints, we define its $M+1$ dimensional ``regret vector'' that is composed of its reward regret~\eqref{def:cumu_rew} and $M$ cost regrets \eqref{def:cumu_cost}. More precisely, considering solely the reward regret (as is done in the RL literature) overlooks the cost expenditures.~Indeed, we show in Theorem~\ref{th:regret_vector_result_1} that the reward regret can be made arbitrary small (with a high probability) at the expense of an increase in the cumulative cost expenditure. Thus, while comparing the performance of two different learning algorithms, we also need to compare their cost expenditures. The reward regret of a learning algorithm is the difference between its reward and the reward of an optimal policy that knows the MDP parameters, while the $i$-th cost regret is the difference between the total cost incurred until $T$ time-steps, and $c^{ub}_iT$. 
	\item We ask the following question in the constrained setup:~\emph{What is the set of ``achievable'' $M+1$ dimensional regret vectors?} In Theorem~\ref{regret:total} we show that the components of the regret vector of UCRL-CMDP, can be bounded as $\tilde{O}(T^{2\slash 3})$.
	\item We show that the use of BOFU allows us to overcome the shortcomings of the CE approach that were encountered in~\cite{altman1991adaptive}, i.e., there are arbitrarily long time-durations during which the CMDP in which the system dynamics are described by the current empirical estimates of transition probabilities is infeasible, and hence the agent is unable to utilize these estimates in order to make control decisions. As a by-product, BOFU also allows us to get rid of ``forced explorations,'' that were utilized in~\cite{altman1991adaptive}, i.e., employing randomized controls occasionally.%~More specifically, in Theorem~\ref{regret:total} we derive bounds on the reward and cost regrets incurred by the UCRL-CMDP algorithm, that hold with a high probability. Since these bounds are $\tilde{O}(\sqrt{T})$, the asymptotic regrets normalized by the operating time horizon, are $0$.
	\item Analogous to the unconstrained RL setup, in which one is interested in quantifying a lower bound on the regret of any learning algorithm, we provide a partial characterization of the set of those $M+1$-dimensional regret vectors, which cannot be achieved under any learning algorithm. More specifically,~in Theorem~\ref{th:lower_bound} we show that a weighted sum of the $M+1$ regrets is necessarily greater than $O\left( D(p)S\sqrt{AT\log(T)}\right)$, where $D(p)$ is the diameter of the underlying MDP, and $S,A$ is the number of states and control actions respectively.
	\item In many applications, an agent is more sensitive to the cost expenditures of some specific resources as compared to the rest, and a procedure to ``tune'' the $M+1$ dimensional regret vector is essential. In Section~\ref{sec:tune_regret}, we consider the scenario in which the agent can pre-specify the desired bounds on each component of the cost regret vector, and introduce a modification to the UCRL-CMDP that allows the agent to keep the cost regrets below these bounds.   
	%    \item We make the unichain assumption while analyzing the performance of our algorithms. In contrast, UCRL2 of~\cite{jaksch2010near} for unconstrained MDPs provides $O(\sqrt{T\log T})$ guarantees on the reward regret under the weaker assumption of communicating MDPs. We impose a stricter assumption on the structure of the underlying MDP due to the consideration of constraints on costs; and it remains to be seen if this can be weakened.
\end{enumerate}
\subsection{Failure of OFU in constrained RL problems}\label{subsec:failure}
Consider a two-state $\cS=\{1,2\}$, two-action $\cA=\{0,1\}$ MDP %shown in Fig.~\ref{fig:mdp} 
in which the controlled transition probabilities $p(1,1,1)=1-\theta$ and $p(1,1,2)=\theta$ are unknown, while remaining probabilitites are equal to $.5$. Assume that $r(1,a),c(1,a)\equiv 0$ and $r(2,a),c(2,a)\equiv 1$, i.e., reward and cost depend only upon the current state. Assume that $\theta>.5$, and the average cost threshold satisfies $c^{ub}<2\theta\slash (1+2\theta)$. Since state $2$ yields reward at the maximum rate, and $\theta>.5$ this means that the optimal action in state $1$ is $1$. Let $\hat{\theta}_t$ and $\epsilon_t$ denote the empirical estimate of $\theta$, and the radius of confidence interval respectively at time $t$. Then UCRL2 sets the optimistic estimate of $\theta$ equal to $\hat{\theta}_t + \epsilon_t$ and then implements the control that is optimal when true parameter value is equal to this estimate. Thus, if $\hat{\theta}_t + \epsilon_t \ge .5$, then it chooses action $1$ in state $1$. Since with a high probability we have $\hat{\theta}_t + \epsilon_t \ge \theta$, and
$\hat{\theta}_t + \epsilon_t  \to \theta$ as $T\to\infty$~\cite{jaksch2010near}, we have that when the index of the RL episode is sufficiently large, the agent implements action $1$ in state $1$. Since the average cost of this policy is $2\theta\slash (1+2\theta)$, this means that UCRL2 violates the average cost constraint. %This is shown in Fig.~\ref{fig:example}. 
%\begin{figure}[h]
%	\centering
%	\subfloat[A $2$ state, $2$ action MDP. $s$ denotes the state value, while $r,c$ denote the cost and reward respectively associated with that state.~Outgoing arrows represent possible controlled state transitions. $p(1,1,2)=\theta$. Parameter $\theta>.5$, and hence it is optimal to take action $a=1$ in state $s=1$.]{ 
	%		\includegraphics[width=5cm]{figures/mdp.pdf}\label{fig:mdp}
	%	}\\
%	\subfloat[$\pi$ is the policy that uses $a=1$ in state $1$. 
%	Under UCRL2 strategy, the optimistic estimate $\hat{\theta}_t +\epsilon_t$ approaches $\theta$ from the right, and the associated cost incurred is more than $c^{ub}$.]{ 
	%		\includegraphics[width=5cm]{figures/example.pdf}\label{fig:example}
	%	}\\
%	\caption{UCRL2 algorithm fails to ensure that the cost constraints are satisfied for the MDP shown in Fig. 1a. }
%	\label{figure:3}
%	%	\label{s1}
%\end{figure}

\section{Preliminaries}\label{sec:prelim}
%Consider a controlled Markov process~\citep{puterman2014markov} $s_t,~t=1,2,\ldots,T$. At each discrete time $t$, an agent applies control $a_t$. State-space, and action space are denoted by $\cS$ and $\cA$ respectively, and are assumed to be finite. 
%The controlled transition probabilities are denoted $p:= \{ p(s,a,s^{\prime}) : s,s^{\prime}\in \cS, a\in\cA\}$. Thus, $p(s,a,s\up)$ is the probability that system state transitions to state $s\up$ upon applying action $a$ in state $s$. 

In our setup, at each time $t$ the agent earns a reward and also incurs $M$ costs. Reward and cost functions are denoted by $r,\{c_i\}_{i=1}^{M}, \cS\times\cA\mapsto\bR$. Thus, the instantaneous reward obtained upon taking an action $a$ in the state $s$ is equal to $r(s,a)$, while the $i$-th cost is equal to $c_i(s,a)$. A controlled Markov process in which the agent earns reward and incurs $M$ costs is defined by the tuple $\mathcal{CMP}= (\cS,\cA,p,r,c_1,c_2,\ldots,c_M)$.
The probabilities $p(s,a,s\up)$ are not known to the agent, while the reward and cost functions $r,\{c_i\}_{i=1}^{M}, \cS\times\cA\mapsto\bR$ are known to the agent. We will now briefly discuss some notions and results on MDPs. 

Let $P^{(t)}_{\pi,p,x}$ denote the $t$-step probability distribution when the policy $\pi$ is applied to the MDP $p$ and the initial state is $x$, while $P_{\pi,p}$ be the corresponding stationary measure \footnote{Under the assumption that a unique stationary measure exists.}. For two measures $\mu_1,\mu_2$, we let $\|\mu_1 -\mu_2\|_{V}$ denote the total variation distance~\cite{villani2008optimal} between $\mu_1$ and $\mu_{2}$. 
\begin{definition}(Unichain MDP)\label{def:unichain}
	%Consider the controlled Markov process $s_t$ that starts in state $s$, and evolves under the application of a stationary policy $\pi$. Let $\bE_{\pi}(T_{s,s^{\prime}})$ denote the expected value of the time taken by the process $s_t$ to hit the state $s^{\prime}$. 
	The MDP $p$ is unichain if under any stationary policy there is a single recurrent class. If an MDP is unichain~\cite{puterman2014markov}, then for the Markov chain induced by any stationary policy $\pi$, we have
	\begin{align}\label{condition:unichain}
		\|P^{(t)}_{\pi,p,x} - P_{\pi,p}  \|_{TV} \le C \rho^{t},~\forall x \in \cS,
	\end{align}
	where $C>0,\alpha<1$ are constants. The mixing time of an MDP $p$ is defined as $T_M(p) := \max_{\pi} \bE_{\pi,p} T_{s,s\up}$, where $T_{s,s\up}$ denotes the time taken by the Markov chain induced by policy $\pi$ to hit state $s\up$, when it starts in state $s$.
	%$\square$
\end{definition}
\begin{definition}(Control Policy)
	Let 
	$$
	\Delta(\cA):=\left\{ \bm{x} \in \bR^{|\cA|}: \sum_{i=1}^{|\cA|} x_i =1, x_i \ge 0  \right\}
	$$ 
	be the $|\cA|$-simplex and $\cF_t$ denote the sigma-algebra~\cite{resnick2019probability} generated by the random variables  $\left\{(s_{\ell},a_\ell)\right\}_{\ell=1}^{t-1}\cup s_t$. A control policy $\pi$~\cite{kumar2015stochastic,puterman2014markov} is a collection of maps $\cF_t\mapsto \Delta(\cA), t=1,2,\ldots$ that chooses action $a_t$ on the basis of past operational history of the system. Thus, under policy $\pi$, we have that $a_t$ is chosen according to the probability distribution $\pi(\cF_t)$. A general control policy is allowed to be history-dependent and randomized. 
\end{definition}
\begin{definition}(Stationary Policy)
	A stationary policy $\pi:\cS\mapsto\Delta(\cA)$, is a mapping from state $\cS$ to a probability distribution on the action space $\cA$, and prescribes randomized controls on the basis of the current state $s_t$. Thus, under policy $\pi$, we have that $a_t$ is chosen according to the probability distribution $\pi(\cdot|s_t)$. 
\end{definition}
%\begin{assumption}\label{assum:recurrent}
%	The MDP $p$ is unichain.
%\end{assumption}
\subsection{Notation}
Throughout, we use bold font for denoting vectors; for example the vector $(x_1,x_2,\ldots,x_N)$ is denoted by $\bm{x}$. We use $\mathbb{N}$ to denote the set of natural numbers, $\bR^{M}$ to denote the $M$ dimensional Euclidean space, and $\bR^{M}_{+}$ to denote non-negative orthant of $\bR^{M}$. Inequalities between two vectors are to be understood component-wise. If $\cE$ is an event~\cite{resnick2019probability}, then $\id(\cE)$ denotes its indicator function. For a control policy $\pi$,\footnote{Assuming that the limit exists.} 
$$
\bar{r}(\pi):=\lim_{T\to\infty}\frac{1}{T} \bE_{\pi} \sum_{t=1}^{T} r(s_t,a_t)
$$ 
denotes its average reward, and 
$$
\bar{c}_i(\pi):= \lim_{T\to\infty}\frac{1}{T} \bE_{\pi} \sum_{t=1}^{T} c_i(s_t,a_t)$$
denotes its average $i$-th cost. For $\bm{x}\in\bR^{N}$, we let $\|\bm{x}\|_1$ denote its $1$-norm. $\bm{0}_{M}$ denotes the $M$-dimensional zero vector consisting of all zeros. For $x,y\in\bR$, we let $x\vee y:= \max \{ x,y\}$.~Throughout, we abbreviate $[M]:= \{1,2,\ldots,M\}$, $S:=|\cS|, A:= |\cA|$. %{\color{red}If $x$ and $y$ are two vectors, then $x\cdot y$ denotes their dot product. We let $SR(s;\mu)$ be the vector that contains the action probabilities associated with $SR(\mu)$ in state $s$.}
\subsection{Constrained MDPs}\label{subsec:cmdp}
We now present some definitions and standard results pertaining to constrained MDPs. These can be found in~\cite{altman}.
\begin{definition}[Occupation Measure]
	Consider the controlled Markov process $s_t$ evolving under the application of a stationary policy $\pi$. Its occupation measure $
	\mu_{\pi}=\left\{ \mu_{\pi}(s,a):(s,a)\in\cS\times\cA\right\}$
	is defined as
	\nal{\mu_{\pi}(s,a):= \lim_{T\to\infty}\frac{1}{T} \bE_{\pi}\left(\sum_{t=1}^{T}\id\left(s_t=s,a_t=a\right)\right),
	}
	and describes the average amount of time that the process $(s_t,a_t)$ spends on each possible state-action pair.
\end{definition}
\begin{definition}[$SR(\mu)$]
	Consider a vector $\mu= \left\{\mu(s,a): (s,a)\in\cS\times\cA\right\}$ that satisfies the constraints~\eqref{lp_c1} and~\eqref{lp_c3} below. Define $SR(\mu)$ to be the following stationary randomized policy. When the state $s_t$ of the environment is equal to $s$, the policy chooses the action $a$ with a probability equal to $\frac{\mu(s,a)}{\sum_{a\up\in\cA }\mu(s,a^{\prime})}$ if $\sum_{a\up\in\cA }\mu(s,a^{\prime})>0$. However, if  $\sum_{a\up\in\cA }\mu(s,a^{\prime})=0$, then the policy takes an action according to some pre-specified rule (e.g. implement $a_t=0$).
\end{definition}
%\begin{definition}[Policy Time Sharing (PTS)]
%Let $\alpha = \{\alpha_{\pi}:\pi \in \Pi_{SD}\}$, where $\alpha_{\pi}\ge 0$ and $\sum_{\pi}\alpha_{\pi} =1$. A PTS policy uses a constant policy from $\Pi_{SD}$ within each cycle, and also the fraction of cycles during which $\pi$ is used converges to $\alpha_{\pi}$ almost surely. 	
%\end{definition}
%\begin{definition}
%	Doeblin's condition: 
%\end{definition}

Consider the controlled Markov process $\mathcal{CMP}= (\cS,\cA,p,r,c_1,c_2,\ldots,c_M)$. The following dynamic optimization problem is a constrained Markov Decision Process (CMDP)~\cite{altman}, 
\begin{align}
	\max_{\pi} &\liminf_{T\to\infty }\frac{1}{T} \bE_{\pi}\sum_{t=1}^{T} r( s_t,a_t) \label{eq:adapt_obj}\\
	\mbox{ s.t. }&  \limsup_{T\to\infty }\frac{1}{T} \bE_{\pi}\sum_{t=1}^{T} c_i( s_t,a_t ) \leq c^{ub}_i\label{eq:adapt_constr}, i\in [M],
\end{align} %{eq:adapt_obj}  {eq:adapt_constr}
where the maximization above is over the class of all history-dependent policies, and $c^{ub}_i$ denotes the desired upper-bound on the average value of $i$-th cost expense. The optimal average reward rate of the CMDP is equal to the optimal value of the above LP, and is denoted by $r\ust$.\par 
\emph{Linear Programming (LP) approach for solving CMDPs}: When the controlled transition probabilities $p(s,a,s\up)$ are known, and $p$ is unichain, an optimal policy for the CMDP~\eqref{eq:adapt_obj}-\eqref{eq:adapt_constr} can be obtained by solving the following linear program (LP)~\cite{altman},
\begin{align}
	&\max_{\mu=\{\mu(s,a):(s,a)\in\cS\times\cA \}} \sum_{(s,a)\in\cS\times\cA} \mu(s,a)r(s,a),\label{lp:1}\\
	&\mbox{ s.t. } \sum_{(s,a)\in\cS\times\cA} \mu(s,a)c_i(s,a) \le c^{ub}_i, i\in [M] \label{lp_c2}\\
	& \sum_{a\in\cA}\mu(s,a) = \sum_{(s^{\prime},b)\in\cS\times\cA}\mu(s^{\prime},b) p(s^{\prime},b,s),~~\forall s\in\cS,\label{lp_c1}\\
	&\mu(s,a)\ge 0,~~\forall (s,a)\in\cS\times\cA, \sum_{(s,a)\in\cS\times\cA} \mu(s,a) = 1\label{lp_c3}.
\end{align}  % 
Let $\mu\ust$ be a solution of the above LP. Then, the stationary randomized policy $SR(\mu\ust)$ solves~\eqref{eq:adapt_obj}-\eqref{eq:adapt_constr}. Moreover, it can be shown that the average reward and $M$ costs of $SR(\mu\ust)$ are independent of the initial starting state $s_0$ if the MDP is unichain~\cite{altman}.%\textcolor{red}{under what assumption?}. 
\subsection{Learning Algorithms and Regret Vector}
We will develop reinforcement learning algorithms to solve the finite-time horizon version of the CMDP~\eqref{eq:adapt_obj}-\eqref{eq:adapt_constr} when the probabilities $p(s,a,s\up)$ are not known to the agent. Let $\cF_t$ denote the sigma-algebra~\cite{resnick2019probability} generated by the random variables  $\left\{(s_{\ell},a_\ell)\right\}_{\ell=1}^{t-1}\cup s_t$. A learning policy $\pi$ is a collection of maps $\cF_t\mapsto \Delta(\cA), t=1,2,\ldots$ that chooses action $a_t$ on the basis of past operational history of the system. In order to measure the performance of a learning algorithm, we define its reward and cost regrets.~The ``cumulative reward regret'' until time $T$, denoted by $\Delta\ur(T)$, is defined as,
\begin{align}\label{def:cumu_rew}
	\Delta\ur(T):=  r^{\star}~T  -  \sum_{t=1 }^{T} r(s_t,a_t),
\end{align} 
where $r^{\star}$ is the optimal average reward of the CMDP~\eqref{eq:adapt_obj}-\eqref{eq:adapt_constr} when controlled transition probabilities $p(s,a,s\up)$ are known. Note that $r\ust$ is the optimal value of the LP~\eqref{lp:1}-\eqref{lp_c3}.~The ``cumulative cost regret" for the $i$-th cost until time $T$ is denoted by $\Delta\uci(T)$, and is defined as, 
\begin{align}\label{def:cumu_cost}
	\Delta\uci(T) :=  \sum_{t=1 }^{T} c_i(s_t,a_t)  -  c^{ub}_i~T.
\end{align} 
\begin{remark}
	In the conventional regret analysis of RL algorithms, the objective is to bound the reward regret $\Delta\ur(T)$. However, in our setup, due to considerations on the cost expenditures, we also need to bound the cost regrets $\Delta\uci(T)$. Indeed, as shown in the Section~\ref{sec:tune_regret}, we can force $\Delta\ur(T)$ to be arbitrarily small at the expense of increased cost regrets, and also vice versa. The consideration of the regret vector, and the possibility of tuning its various components, is a key novelty of our work. The problem of tuning this vector is challenging because its various components are correlated.
\end{remark}
%\begin{definition}\label{def:eps_opt}
%	Let $b\in \bR,\bm{b}\in \bR^{M}$.~If a policy $\pi$ satisfies $ \bar{r}(\pi,p)\ge r\ust - b$, and $\bar{c}_i(\pi,p) \le c^{ub}_i +b_i,~\forall i\in [M]$, we say it is $(b,\bm{b})$-optimal. Otherwise, we say it is $(b,\bm{b})$-suboptimal. 
%\end{definition}
%
%While comparing the cumulative regret of a learning algorithm over $T$ steps, with a $(b,\bm{b})$-optimal policy, we consider the modified regrets given as follows,
%	$$
%	\Delta\ur_{b,\bm{b}}(T):= \Delta\ur(T) - bT,
%	$$
%	$$
%	\Delta\uci_{b,\bm{b}}(T):= \Delta\uci(T)-b_iT,~i\in [M].
%	$$

\section{UCRL-CMDP: A Learning Algorithm for CMDPs}\label{sec:ucb_cmdp}

\begin{algorithm}
	\begin{algorithmic}
		\caption{UCRL-CMDP}
		\label{algo:ucb_cmdp}
		\State {\bfseries Input:} State-space $\cS$, Action-space $\cA$, Confidence parameter $\delta$, Time horizon $T$\\
		\State {\bfseries Initialize:} Set $t:=1$, and observe the initial state $s_1$.
		\For{ Episodes $k=1,2,\ldots$ }
		\State 
		\textbf{Initialize Episode }$\bm{k}$:
		\begin{enumerate}
			\item Set the start time of episode $k$, $\tau_k:=t$. For all state-action tuples $(s,a)\in\cS\times\cA$, initialize the number of visits within episode $k$, $n_{k}(s,a)=0$. 
			\item For all $(s,a)\in\cS\times\cA$ set $N_{\tau_k}(s,a)$, i.e., the number of visits to $(s,a)$ prior to episode $k$. Also set the transition counts $N_{\tau_k}(s,a,s^{\prime})$ for all $(s,a,s^{\prime})\in\cS\times\cA\times\cS$.
			\item Compute the empirical estimate $\hat{p}_t$ of the MDP as in~\eqref{eq:empirical_estimate_mdp}.
		\end{enumerate}
		\textbf{Compute Policy }$\bm{\tilde{\pi}_k}$:
		\begin{enumerate}
			\item Let $\cC_{\tau_k}$ be the set of plausible MDPs as in~\eqref{eq:ci_def}.
			\item Solve~\eqref{eq:ucb_cmdp_obj}-\eqref{eq:ucb_cmdp_c3} to obtain $\tilde{\pi}_k$. %by using the iterative algorithm of Section~\ref{appendix:primal_dual}.
			\item In case ~\eqref{eq:ucb_cmdp_obj}-\eqref{eq:ucb_cmdp_c3} is infeasible, choose $\tilde{\pi}_k$ to be some pre-determined policy (chosen at time $t=0$).
		\end{enumerate}
		\State \textbf{Implement} $\bm{\tilde{\pi}_k}$:
		\While{$t-\tau_{k_t}< \lceil T^{\alpha}\rceil $}
		\State{
			\begin{enumerate}
				\item Sample $a_t$ according to the distribution $\tilde{\pi}_k(\cdot|s_t)$. Observe reward $r(s_t,a_t)$, and observe next state $s_{t+1}$.
				\item Update $n_k(s_t,a_t)=n_k(s_t,a_t)+1$.
				\item Set $t:=t+1$.
			\end{enumerate}
		}
		\EndWhile
		\EndFor
	\end{algorithmic}
\end{algorithm}
We propose UCRL-CMDP to adaptively control an unknown CMDP. It is depicted in Algorithm~\ref{algo:ucb_cmdp}. UCRL-CMDP maintains empirical estimates of the unknown transition probabilities as follows,~
\alig{
	\hat{p}_{t}(s,a,s^{\prime}) = \frac{N_{t}(s,a,s^{\prime}) }{N_{t}(s,a) \vee 1 }, \forall s,s^{\prime}\in\mathcal{S}, a\in\mathcal{A},\label{eq:empirical_estimate_mdp}
}
~where $N_{t}(s,a)$ and $N_{t}(s,a,s^{\prime})$ 
%\begin{align*}
%N_{t}(s,a) &= \sum_{l<t} \ev\left\{ s_{l}=s,a_{l}=a\right\}, \\
%\mbox{ and } N_{t}(s,a,s^{\prime}) &= \sum_{l<t-1} \ev\left\{ s_{l}=s,a_{l}=a,s_{l +1 }= s^{\prime}\right\}
%\end{align*}
denote the number of visits to $(s,a)$ and $(s,a,s\up)$ until $t$ respectively.

\emph{Confidence Intervals}: Additionally, it also maintains confidence interval $\cC_t$ associated with the estimate $\hat{p}_t$ as follows,
\begin{align}\label{eq:ci_def}
	& \cC_t  := \notag\\
	&\Big\{p\up :
	| p\up(s,a,s\up) -\hat{p}_{t}(s,a,s\up)| \leq \epsilon_t(s,a), \forall (s,a) \Big\},
\end{align}
where
\begin{align}\label{eq:ci}
	\epsilon_t(s,a):=  \sqrt{ \frac{2\log\left(T^{b} |\cS||\cA|\right) }{N_t(s,a) \vee 1 } },
\end{align}
$b>1$ is a constant.
\\
\emph{Episode}: UCRL-CMDP proceeds in episodes, and utilizes a single stationary control policy within an episode. Each episode is of duration $\lceil T^{\alpha}\rceil$ steps.
%A new episode begins each time the number of visits to some state-action pair $(s,a)$ doubles. 
Let $\tau_k$ denote the start time of episode $k$. $k$-th episode is denoted by $\mathcal{E}_k := \left\{\tau_k,\tau_k +1,\ldots,\tau_{k+1} -1\right\}$, and comprises of $\tau_{k+1}-\tau_{k}$ consecutive time-steps. Denote by $k_t$ the index of the ongoing episode at time $t$.~At the beginning of $\cE_k$, the agent solves the following \emph{constrained} optimization problem in which the decision variables are
(i) Occupation measure $\mu=\{\mu(s,a):(s,a)\in\cS\times\cA\}$ of the controlled process, and (ii) ``Candidate'' MDP $p^{\prime}$,
\begin{align}
	&\max_{\mu,p^{\prime}} \sum_{(s,a)\in\cS\times\cA} \mu(s,a)r(s,a),\label{eq:ucb_cmdp_obj}\\
	&\mbox{ s.t. }  \sum_{(s,a)\in\cS\times\cA} \mu(s,a)c_i(s,a) \le c^{ub}_i, i\in [M] \label{eq:ucb_cmdp_c2}\\
	&\sum_{a\in\cA}\mu(s,a) = \sum_{(s^{\prime},b)}\mu(s^{\prime},b) p^{\prime}(s^{\prime},b,s),~~\forall s\in\cS,\label{eq:ucb_cmdp_c1}\\
	&\mu(s,a)\ge 0~~\forall (s,a), \sum_{(s,a)} \mu(s,a) = 1,\\
	&~~\qquad p\up \in \cC_{\tau_{k}}. \label{eq:ucb_cmdp_c3} %  \|p^{\prime}(s,a,\cdot) - \hat{p}_{\tau_k}(s,a,\cdot)\|_1 \le \epsilon_{\tau_k}(s,a)
\end{align}    %{eq:ucb_cmdp_obj}-{eq:ucb_cmdp_c3}
The maximization w.r.t. $p^{\prime}$ denotes that the agent is optimistic regarding the belief of the ``true'' (but unknown) MDP $p$, while that w.r.t. $\mu$ ensures that the agent optimizes its control strategy for this optimistic MDP. The constraints~\eqref{eq:ucb_cmdp_c2} ensure that the cost expenditures do not exceed the thresholds $\{c^{ub}_i\}_{i=1}^{M}$, and hence ensure that the agent also balances the cost expenses while being optimistic with respect to the rewards about the choice of the MDP thereby taking a balanced approach to optimism when the underlying MDP parameters are unknown. If the constraints~\eqref{eq:ucb_cmdp_c2} were absent, we would recover the UCRL2 algorithm of~\cite{jaksch2010near} that is based on the OFU principle~\cite{lai1985asymptotically,agrawal1995sample}. However, as is shown in Section~\ref{subsec:failure}, the OFU principle might fail when it is applied for learning the optimal controls for CMDPs. Indeed, as is shown in the example in Section~\ref{subsec:failure}, the limiting average cost is greater than the threshold value of cost. The BOFU principle proposed in this work is a natural extension of the OFU principle to the case when the agent has to satisfy certain constraints on costs, in addition to maximizing the rewards.~In case the problem~\eqref{eq:ucb_cmdp_obj}-\eqref{eq:ucb_cmdp_c3} is feasible, let $(\tilde{\mu}_k,\tilde{p}_k)$ denote a solution. The agent then chooses $a_t$ according to $SR(\tilde{\mu}_k)$ within $\cE_k$. However, in the event the LP~\eqref{eq:ucb_cmdp_obj}-\eqref{eq:ucb_cmdp_c3} is infeasible, the agent implements an arbitrary stationary control policy that has been chosen at time $t=0$. In summary, it implements a stationary controller within $\cE_k$, which is denoted by $\tilde{\pi}_k$. We make the following assumptions on the MDP $p$ while analyzing UCRL-CMDP.
\begin{assumption}\label{assum:1}
	\begin{enumerate}
		\item The MDP $p=\left\{p(s,a,s^{\prime}) : s,s^{\prime}\in \cS,a\in\cA  \right\}$ is unichain. Thus, under a stationary policy $\pi$ we have
		\begin{align}\label{condition:unichain_1}
			\|P^{(t)}_{\pi,p,x} - P_{\pi,p}  \|_{TV} \le C \rho^{t},~t=1,2,\ldots,
		\end{align}
		where $C>0,0\le \rho<1$.
		\item The CMDP~\eqref{eq:adapt_obj}-\eqref{eq:adapt_constr} is feasible, i.e., there exists a policy under which the average cost constraints~\eqref{eq:adapt_constr} are satisfied.
		\item  Without loss of generality, we assume that the magnitude of rewards and costs are upper-bounded by $1$, i.e.,
		\begin{align*}
			|r(\cdot,\cdot)|,| c_i(\cdot,\cdot) | <1.
		\end{align*}
		Hence, if $r^{\star}$ denotes optimal reward rate of~\eqref{eq:adapt_obj}-\eqref{eq:adapt_constr}, then $r^{\star}<1$. Moreover, the cost bounds $\{c^{ub}_i\}_{i=1}^{M}$ can be taken to be less than $1$.  
	\end{enumerate}
\end{assumption}

We establish the following bound on the regrets of UCRL-CMDP.
\begin{theorem}\label{regret:total}
	Consider the UCRL-CMDP (Algorithm~\ref{algo:ucb_cmdp}) applied with $\delta = 1\slash T$ to an MDP $p$ that satisfies Assumption~\ref{assum:1}. The reward and cost regrets can be bounded as follows,
	\begin{align}
		& \bE \Delta\ur(T),\bE \Delta\uci(T),~i\in [M] \notag\\
		&\le 4 T_M \Bigg((\sqrt{2}+1)\sqrt{SAT}+ T^{\beta} \sqrt{\log\Bigg( \frac{SAT}{\delta}\Bigg)}\Bigg)\notag \\
		&+\frac{C\lceil T^{1-\alpha}\rceil}{1-\rho}+\delta T + \frac{2}{T^{2b-2}|\cS||\cA|},
	\end{align}
	where $\beta > 1\slash 2$ satisfies $2\beta - \alpha = 1$. Upon letting $\alpha = 1\slash 3$ in the above, the above reduces to
	\begin{align}
		& \bE \Delta\ur(T),\bE \Delta\uci(T),~i\in [M] \notag\\
		&\le 4 T_M\Bigg((\sqrt{2}+1)\sqrt{SAT}+ T^{2\slash 3} \sqrt{\log\Bigg( \frac{SAT}{\delta}\Bigg)}\Bigg)\notag \\
		&+\frac{C\lceil T^{2\slash 3}\rceil}{1-\rho}+\delta T + \frac{2}{T^{2b-2}|\cS||\cA|},
	\end{align}
\end{theorem}

\section{Proof of Theorem~\ref{regret:total}}
We begin by introducing few notation.~If $\cB$ denotes a subset of $\cS$, then we let $\Pi_{\cB}$ be the set of those policy $\pi$ for which the occupation measure $\mu_{\pi}$ satisfies $\mu_{\pi}(s) >0$ for all $s \in \cB$. Let $\cB_{\pi}$ denote the set of states for which $\mu_{\pi}(s)>0$.~We now derive few preliminary results that are used while proving the main result.

The following result can be shown by an application of Azuma-Hoeffding inequality~\cite{azuma1967weighted}.
\begin{lemma}\label{lemma:conc}
	$$
	\bP(p\in \cC_{\tau_k}) > 1 - \frac{1}{T^{b-1}},
	$$
	where the confidence ball $\cC_{\tau_k}$ is as in~\eqref{eq:ci_def}. Define the set $\cG_1 := \{ p\in \cC_{\tau_k},~\forall k =1,2,\ldots,K \}$. Then,
	$$
	\bP(\cG_1) \ge 1 - \frac{1}{T^{b-1-(1-\alpha)}}.
	$$
\end{lemma}
\begin{IEEEproof}
If the number of visits to $(s,a)$ were determinisic, denoted by $N(s,a)$, and $N(s,a,s\up)$ the corresponding number of visits to $(s,a,s\up)$, then it follows from Azuma-Hoeffding's inequality~\cite{azuma1967weighted} that
\nal{
\bP(| N_{t}(s,a) p(s,a,s\up) -N_{t}(s,a,s\up) | \ge \eps) \le \exp\left(-\frac{\eps^2}{2N(s,a)}\right). 
}	
Upon letting $\eps$ equal to $\sqrt{2N(s,a)\log\left(T^{b} |\cS||\cA|\right)  }$ in the above, we get that 
$|p(s,a,s\up)-p(s,a,s\up|\le  \sqrt{ \frac{2\log\left(T^{b} |\cS||\cA|\right) }{N(s,a) \vee 1 } }$ holds with a probability greater than $\frac{1}{T^{b} |\cS||\cA|}$. The proof is then completed by considering all possibilities for $N(s,a)$, state-action pairs and using union bound. 
\end{IEEEproof}

\begin{lemma}\label{lemma:conc_visits}
Define 
\alig{
	& \cG_2 := \Bigg\{ \sum_{k=1}^{K} \frac{n_\ell(s,a) - \bE(n_\ell(s,a) | \cF_{\tau_\ell})}{\sqrt{N_{\ell}(s,a) }} \le T^{\beta} \sqrt{\log\Bigg( \frac{SAT}{\delta}\Bigg)},\notag\\
	& \forall (s,a) \in \cS \times \cA\Bigg\},\label{def:g2}
}
where $K$ is the total number of episodes and $\beta > 1\slash 2$ satisfies $2\beta - \alpha = 1$.~We have
$$
\bP(\cG_2) \ge 1-\delta.
$$
\end{lemma}
\begin{IEEEproof}
	Note that $\frac{ n_k(s,a) - \bE(n_k(s,a) | \cF_{\tau_k})}{\sqrt{N_{k}(s,a)}},~k =1,2,\ldots,K$ is a martingale difference sequence. Furthermore, since the duration of each episode is bounded by $T^{\alpha}$, and $\sqrt{N_{k}(s,a)}\ge 1$, we have $\frac{ n_k(s,a) - \bE(n_k(s,a) | \cF_{\tau_k})}{\sqrt{N_{k}(s,a)}}\le T^{\alpha}$.~By applying Azuma-Hoeffding's inequality to this martingale difference sequence, we get that the probability of the event $\sum_{k=1}^{K}\frac{ n_k(s,a) - \bE(n_k(s,a) | \cF_{\tau_k})}{\sqrt{N_{k}(s,a)}}\ge T^{\beta} \sqrt{\log\Bigg( \frac{SAT}{\delta}\Bigg)}$ can be upper-bounded by 
	$$
	\exp\left( - \frac{T^{2\beta}\log\Bigg( \frac{SAT}{\delta}\Bigg)}{ T^{1-\alpha} T^{2\alpha} }  \right)=\exp\left( -T^{2\beta-(1+\alpha)}\log\Bigg( \frac{SAT}{\delta}\Bigg)\right).
	$$
	Since $2\beta -(1+\alpha)=0$, the above bound reduces to $\frac{\delta}{SAT}$.~The proof then follows by using union bound for all state-action pairs and $k$.
\end{IEEEproof}

\begin{lemma}\label{lemma:suff_visits}
	If $s\in \cB_{\pi_k}$, then 
	$$
	\bE\left\{ n_k(s,a) | \cF_{\tau_k}\right\}\ge \Bigg\lfloor \frac{T^{\alpha}}{2T_M}\Bigg\rfloor \times \frac{\pi_k(a|s)}{2}.
	$$
\end{lemma}

\begin{IEEEproof}
Since we have $\bE_{\pi} T_{s\up,s}\le T_M,~\forall s\up \in \cS$, it follows from Markov's inequality that the probability with which $s_t$ does not hit the state $s$ in $2T_M$ steps, is less than $1\slash2$, or equivalently the state $s$ is visited atleast once with a probability greater than $1\slash 2$,~which yields us the following,
\alig{
	\min_{s\up\in\cS} \bE_{\pi} \Big\{\sum_{t=1}^{\lceil 2T_M\rceil} \id\{s_t = s\} | s_0 = s\up \Big\} \ge \frac{1}{2}.\label{eq:visit_ineq}
}

 The proof is then completed by dividing the total time of $\lceil T^{\alpha} \rceil$ steps in an episode into ``mini-episodes'' of $\lceil 2T_M\rceil$ steps each, and noting that $n_k(s,a)$ is the sum of number of visits to $(s,a)$ during each such mini-episode.
\end{IEEEproof}
We begin by giving an equivalent characterization of the UCRL-CMDP rule.~At each $\tau_k$, it assigns an index $\cI_k(\pi)$ to each stationary policy $\pi$ as follows,
\nal{
\cI_k(\pi) :=  \max_{\te \in \cC_{\tau_k}} \Big\{ \bar{r}(\pi,\te): \bar{c}_i(\pi,\te) \le c^{ub}_i,~i\in [M]\Big\}.
}
In case the above optimization problem is infeasible, i.e. $\bar{c}_i(\pi,\te) > c^{ub}_i,~\forall \te\in\cC_{\tau_{k}}$ for some $i$,~then the policy is assigned an index of $-\infty$. ~It then implements a policy with the largest index during $\cE_k$.%sufficiently small value (so that it won't be implemented).

Define the ``good set" $\cG := \cG_1 \cap \cG_2$. 
\begin{lemma}\label{lemma:bound_ball}
	On the set $\cG$ we have the following for $\te\in\cC_{\tau_{k}}$,
	\alig{
		& |\bar{r}(\pi,p) - \bar{r}(\pi,\te)|, |\bar{c}_i(\pi,p) - \bar{c}_i(\pi,\te)|~i\in [M] \notag\\
		&\le 2\max_{s}\sum_{a\in \cA} \pi(a|s)\epsilon_{\tau_k}(s,a).\label{ineq:2}
	}
\end{lemma}
\begin{IEEEproof}
Note that $P^{(1)}_{\pi,p,s},s\in\cS$ denotes the transition probabilities of the Markov chain that results when the policy $\pi$ is applied to the MDP $p$.~Consider an MDP $\te \in \cC_{\tau_{k}}$.~Since on $\cG$, we have $p  \in \cC_{\tau_{k}}$, we get the following,
\nal{
&\|P^{(1)}_{\pi,\hat{p}_{\tau_{k}},s}- P^{(1)}_{\pi,p,s}\|_{\infty},\|P^{(1)}_{\pi,\hat{p}_{\tau_{k}},s}- P^{(1)}_{\pi,\te,s}\|_{\infty}\\
& \le \sum_{a\in\cA} \pi(a|s)\epsilon_{\tau_k}(s,a),
}
so that from triangle inequality we have that 
\alig{
\|P^{(1)}_{\pi,p,s} - P^{(1)}_{\pi,\te,s}\|_{\infty} \le 2\sum_{a\in \cA} \pi(a|s)\epsilon_{\tau_k}(s,a).\label{ineq:1}
}
\eqref{ineq:2} then follows from Theorem~\ref{th:mitro}.
\end{IEEEproof}
For a policy $\pi$, we denote the quantity $r\ust -\bar{r}(\pi,p)$ by its instantaneous reward regret, and $\bar{c}_i(\pi,p)-c^{ub}_i$ its instantaneous cost regret for the $i$-th cost.~We now show that if the instantaneous reward regret, or an instantaneous cost regret of a policy is greater than a certain threshold (that depends upon the radius of the confidence ball at time $\tau_k$), then it is not played during $\cE_k$.~For a stationary policy $\pi$, define
$$
\delta_{k}(\pi) := 2\max_{s\in \cB_{\pi}} \sum_{a\in \cA} \pi(a|s)\epsilon_{\tau_k}(s,a).
$$

Consider the following two possibilities.

Case A) $\bar{c}_i(\pi,p)>c^{ub}_i+\delta_{k}(\pi)$ for some $i$: From~\eqref{ineq:2} we have that $|\bar{c}_i(\pi,p) - \bar{c}_i(\pi,\te)|\le \delta_{k}(\pi)$ which implies $\bar{c}_i(\pi,\te) > c^{ub}_i$ for all $\te\in \cC_{\tau_{k}}$. Thus $\cI_k(\pi)=-\infty$. 

Case B) From~\eqref{ineq:2} we have that $|\bar{r}(\pi,p) - \bar{r}(\pi,\te)| \le \delta_{k}(\pi)$ for all $\te\in\cC_{\tau_{k}}$, so that the index $\cI_k(\pi)$ is bounded by $\bar{r}(\pi,p)+\delta_{k}(\pi)$. 

The following result summarizes this discussion.

\begin{lemma}\label{lemma:ub}
	Let $\pi$ be a stationary randomized policy. On the set $\cG$ we have that $\cI_k(\pi)=-\infty$ if $\bar{c}_i(\pi,p)>c^{ub}_i+\delta_{k}(\pi),~\mbox{for some } i\in [M]$. Also, $\cI_k(\pi)\le \bar{r}(\pi,p)+\delta_{k}(\pi)$.
\end{lemma}

We now show that if a stationary policy is feasible for the MDP $p$, i.e. $\bar{c}_i(\pi,p)\le c^{ub}_i,~\forall i$, then its index $\cI_k(\pi)$ is lower-bounded by $r\ust$.
\begin{lemma}\label{lemma:index_lb}
	If $\pi$ is feasible for the true MDP, i.e. it satisfies $\bar{c}_i(\pi,p)\le c^{ub}_i,~\forall i\in[M]$, then on $\cG$ its index satisfies $\cI_k(\pi) \ge \bar{r}(\pi,p)$. With $\pi$ set equal to the policy which solves the CMDP $\max_{\pi}  \left\{\bar{r}(\pi,p): \bar{c}_i(\pi,p)\le c^{ub}_i,~\forall i\in[M]\right\}$, we obtain that the index of an optimal policy is greater than $r\ust$.
\end{lemma}
\begin{IEEEproof}
	~Note that on the set $\cG$, the true MDP $p$ always belongs to $\cC_{\tau_{k}}$. If $\bar{c}_i(\pi,p)\le c^{ub}_i,~\forall i\in [M]$, we have
	\nal{
		\cI_k(\pi) & =  \max_{\te \in \cC_{\tau_k}} \Big\{ \bar{r}(\pi,\te): \bar{c}_i(\pi,\te) \le c^{ub}_i,~i\in [M]\Big\}\\
		&\ge \bar{r}(\pi,p).	
	}
\end{IEEEproof}

Upon combining Lemma~\ref{lemma:ub} and Lemma~\ref{lemma:index_lb}, we obtain the following result.
\begin{lemma}\label{lemma:learn}
	On the set $\cG$, the instantaneous regrets during $\cE_k$ can be bounded by $\delta_{k}(\pi_k)$.
\end{lemma}
\begin{IEEEproof}
It follows from Lemma~\ref{lemma:ub} that if $\bar{c}_i(\pi,p)>c^{ub}_i + \delta_{k}(\pi_k)$, then $\cI_k(\pi)=-\infty$. However, it is shown in Lemma~\ref{lemma:index_lb} that there is a policy $\tilde{\pi}$ whose index is greater than $r\ust$. Since index of $\pi$ is less than that of $\tilde{\pi}$, the policy $\pi$ will not be played by UCRL-CMDP. This means that $\bar{c}_i(\pi,p)\le c^{ub}_i + \delta_{k}(\pi_k)$, which shows that the instantaneous cost regret is bounded by $\delta_{k}(\pi_k)$.
	
To bound the reward regret, note that it was shown in Lemma~\ref{lemma:index_lb} that the index of an optimal policy is greater than $r\ust$, and hence the index $\cI_k(\pi_k)$ is greater than $r\ust$.~From Lemma~\ref{lemma:ub} we have $\cI_k(\pi_k)\le \bar{r}(\pi,p)+\delta_{k}(\pi_k)$. Hence we have that $\bar{r}(\pi,p)+\delta_{k}(\pi_k)\ge r\ust$, or $\bar{r}(\pi,p)\ge r\ust - \delta_{k}(\pi_k)$, which shows that the instantaneous reward regret is bounded by $\delta_{k}(\pi_k)$.
\end{IEEEproof}
We now use the result on instantaneous regrets in order to bound the cumulative regrets of UCRL-CMDP.
\begin{IEEEproof}[Proof of Theorem~\ref{regret:total}]
We will only discuss bound on reward regret, since the bound on cost regrets can be derived using similar steps.~Note that the expected regret during $\cE_k$ can be written as follows,
\nal{
&\bE \left( \bE\left\{ \sum_{t\in \cE_k} r\ust - r(s_t,a_t) |\cF_{\tau_k}	\right\}\right)\\
&= \bE \left( \bE\left\{ \sum_{t\in \cE_k} r\ust - \bar{r}(\pi_k,p) \Big|\cF_{\tau_k}	\right\}\right)\\
& + \bE \left( \bE\left\{ \sum_{t\in \cE_k}\bar{r}(\pi_k,p) -r(s_t,a_t) \Big|\cF_{\tau_k}	\right\}\right)\\
&\le \bE \left( \bE\left\{ \sum_{t\in \cE_k} r\ust - \bar{r}(\pi_k,p) \Big|\cF_{\tau_k}	\right\}\right)\\
& + \bE \left( \bE\left\{ \sum_{t=1}^{\infty}|\bar{r}(\pi_k,p) -r(s_t,a_t)| \Big|\cF_{\tau_k}	\right\}\right)\\
& \le \bE \left( \bE\left\{ \sum_{t\in \cE_k} r\ust - \bar{r}(\pi_k,p) \Big|\cF_{\tau_k}	\right\}\right)+\frac{C}{1-\rho},
}
where the last inequality follows from~\eqref{condition:unichain_1}. Denote $\Delta\ur_k:=\bE\left\{ \sum_{t\in \cE_k} r\ust - \bar{r}(\pi_k,p) \Big|\cF_{\tau_k}	\right\}$ as the regret incurred during the $k$-th episode. It follows from the above discussion that the cumulative expected regret can be bounded as follows,
\alig{
\bE \Delta\ur(T) \le  \bE\left(\sum_{k=1}^{K} \Delta\ur_k \right) + K\frac{C}{1-\rho},\label{ineq:reg_dec}
}
where $K$ is the total number of episodes.~Henceforth we will only focus on bounding the first term in the r.h.s. above.~This is bounded separately on the sets $\cG,\cG^{c}_1,\cG^{c}_2$. 

We begin by bounding $\sum_{k=1}^{K} \Delta\ur_k $ on $\cG$. We have,
\alig{
	& \Delta\ur_k\le \delta_{k}(\pi_k) |\cE_k|  \notag\\
	& =  \left[\max_{s:s\in\cB_{\pi_k}} \sum_{a\in\cA} \pi_k(a|s)\frac{ \sqrt{2\log\left(T^{b} |\cS||\cA|\right) }}{\sqrt{N_k(s,a)}}\right] |\cE_k|\notag\\
& \le \sum_{(s,a):s\in\cB_{\pi_k}} |\cE_k|\pi_k(a|s)\frac{ \sqrt{2\log\left(T^{b} |\cS||\cA|\right) }}{\sqrt{N_k(s,a)}} \notag\\
& =4T_M \sum_{(s,a):s\in\cB_{\pi_k}} \frac{|\cE_k|}{2T_M}\frac{1}{2}    \pi_k(a|s)\frac{ \sqrt{2\log\left(T^{b} |\cS||\cA|\right) }}{\sqrt{N_k(s,a)}} \notag\\
& =4T_M \sum_{(s,a):s\in\cB_{\pi_k}} \bE(n_k(s)|\cF_k)  \frac{\pi_k(a|s) \sqrt{2\log\left(T^{b} |\cS||\cA|\right) }}{\sqrt{N_k(s,a)}} \notag\\
& +4T_M \sum_{(s,a):s\in\cB_{\pi_k}} \left\{\frac{|\cE_k|}{2T_M}\frac{1}{2} - \bE(n_k(s)|\cF_k)   \right\}\times \notag\\ 
&\qquad \frac{\pi_k(a|s) \sqrt{2\log\left(T^{b} |\cS||\cA|\right) }}{\sqrt{N_k(s,a)}}\notag\\
&\le 4T_M \sum_{(s,a):s\in\cB_{\pi_k}} \bE(n_k(s)|\cF_k) \frac{ \pi_k(a|s) \sqrt{2\log\left(T^{b} |\cS||\cA|\right) }}{\sqrt{N_k(s,a)}},\label{ineq:4}
}
where the first inequality follows from Lemma~\ref{lemma:learn}, and the last inequality follows from Lemma~\ref{lemma:suff_visits}.~We will now bound the term $\sum_{k=1}^{K} \sum_{(s,a):s\in\cB_{\pi_k}} \frac{\bE(n_k(s)|\cF_k)  \pi_k(a|s)}{\sqrt{N_k(s,a)}}$.~We have
\alig{
& \sum_{k=1}^{K} \frac{\bE(n_k(s)|\cF_k)  \pi_k(a|s)}{\sqrt{N_k(s,a)}} = \sum_{k=1}^{K} \frac{ \bE(n_k(s,a)|\cF_k) }{\sqrt{N_k(s,a)}} \notag\\
& = \sum_{k=1}^{K} \frac{ n_k(s,a) }{\sqrt{N_k(s,a)}}  +  \sum_{k=1}^{K} \frac{ \bE(n_k(s,a)|\cF_k) - n_k(s,a) }{\sqrt{N_k(s,a)}} \label{ineq:6}
}
As is shown in~\cite[p.~1578]{jaksch2010near}, the term $\sum_{k=1}^{K}\sum_{(s,a)} \frac{ n_k(s,a) }{\sqrt{N_k(s,a)}} $ can be bounded by $(\sqrt{2}+1)\sqrt{SAT}$, while from~\eqref{def:g2} we have that on $\cG_2$, the term $ \sum_{k=1}^{K} \frac{ \bE(n_k(s,a)|\cF_k) - n_k(s,a) }{\sqrt{N_k(s,a)}} $ is bounded by $T^{\beta} \sqrt{\log\Bigg( \frac{SAT}{\delta}\Bigg)}$. It follows from~\eqref{ineq:6} and the discussion above that on $\cG$ we have
\alig{
 \sum_{(s,a):s\in\cB_{\pi_k}}&\sum_{k=1}^{K} \frac{\bE(n_k(s)|\cF_k)  \pi_k(a|s)}{\sqrt{N_k(s,a)}} \le (\sqrt{2}+1)\sqrt{SAT} \notag \\
&+ T^{\beta} \sqrt{\log\Bigg( \frac{SAT}{\delta}\Bigg)}.\label{ineq:5}
}
Upon summing~\eqref{ineq:4} over episodes, and using~\eqref{ineq:5}, we obtain that the regret on $\cG$ can be bounded as follows, 
\alig{
\sum_{k=1}^{K} \Delta\ur_k \le 4T_M \Bigg((\sqrt{2}+1)\sqrt{SAT}+ T^{\beta} \sqrt{\log\Bigg( \frac{SAT}{\delta}\Bigg)}\Bigg).\label{ineq:G_bound}
}
This completes the analysis on $\cG$.

We now analyze the regret on $\cG^{c}_2$. From Lemma~\ref{lemma:conc_visits}, the probability of $\cG^{c}_2$ is bounded by $\delta$. On $\cG^{c}_2$, the sample path regret $\sum_{k=1}^{K}\Delta\ur_k$ can be trivially bounded by $T$, so that its contribution to the expected regret is bounded by $\delta T$.

To analyze the regret on $\cG^{c}_1$ we note that if the confidence ball $\cC_{\tau_{k}}$ at time $\tau_k$ fails, then the regret during $\cE_k$ can be bounded by the duration of $\cE_k$. Since $\tau_{k+1}-\tau_k =\lceil T^{\alpha}\rceil$, the regret during $\cE_k$ is bounded by $\lceil T^{\alpha}\rceil$.% {\color{red} Prove this..}.
~From Lemma~\ref{lemma:conc} we have that the probability with which confidence ball fails at time $t$ is upper-bounded by $\frac{2}{T^{2b-1}|\cS||\cA|}$. Hence, the expected regret from the failure of ball (in case an episode starts at $t$) at time $t$ is bounded by $\frac{2\lceil T^{\alpha}\rceil}{T^{2b-1}|\cS||\cA|}$, so that the cumulative expected regret is bounded by $\frac{2}{T^{2b-2}|\cS||\cA|}$.
\end{IEEEproof}

\section{Learning Under Bounds on Cost Regret}\label{sec:tune_regret}
The upper-bounds for the regrets of UCRL-CMDP in Theorem~\ref{regret:total} are the same for reward and $M$ costs regrets. However, in many practical applications, an agent is more sensitive to over-utilizing certain specific costs, as compared to the other costs. Thus, in this section, we derive algorithms which enable the agent to tune the upper-bounds on the regrets of different costs.~We also quantify the reward regret of these algorithms.
\subsection{Modified UCRL-CMDP}
Throughout this section we assume that $p$ satisfies the following.
\begin{assumption}\label{assum:2}
	For the MDP $p$, there exists a stationary policy under which the average costs are strictly below the thresholds $\{c^{ub}_i:i=1,2,\ldots,M\}$. More precisely, there exists an $\epsilon>0$ and a stationary policy $\pi_{feas.}$ such that we have $\bar{c}_i(\pi_{feas.})<c^{ub}_i-\epsilon, \forall i\in [M]$. Define
	\begin{align}\label{eq:eta_def}
		\eta := \min_{i\in [M]} \left\{c_i^{ub} - \epsilon-\bar{c}_i(\pi_{feas.}) \right\}.
	\end{align}
\end{assumption}

\begin{algorithm}
	\begin{algorithmic}
		\caption{Modified UCRL-CMDP}
		\label{algo:ucb_cmdp_modi}
		\State {\bfseries Input:} State-space $\cS$, Action-space $\cA$, Confidence parameter $\delta$, Time horizon $T$\\
		\State {\bfseries Initialize:} Set $t:=1$, and observe the initial state $s_1$.
		\For{ Episodes $k=1,2,\ldots$ }
		\State 
		\textbf{Initialize Episode }$\bm{k}$:
		\begin{enumerate}
			\item Set the start time of episode $k$, $\tau_k:=t$. For all state-action tuples $(s,a)\in\cS\times\cA$, initialize the number of visits within episode $k$, $n_{k}(s,a)=0$. 
			\item For all $(s,a)\in\cS\times\cA$ set $N_{\tau_k}(s,a)$, i.e., the number of visits to $(s,a)$ prior to episode $k$. Also set the transition counts $N_{\tau_k}(s,a,s^{\prime})$ for all $(s,a,s^{\prime})\in\cS\times\cA\times\cS$.
			\item Compute the empirical estimate $\hat{p}_t$ of the MDP as in~\eqref{eq:empirical_estimate_mdp}.
		\end{enumerate}
		\textbf{Compute Policy }$\bm{\tilde{\pi}_k}$:
		\begin{enumerate}
			\item Let $\cC_{\tau_k}$ be the set of plausible MDPs as in~\eqref{eq:ci_def}.
			\item Solve~\eqref{eq:ucb_mod_cmdp_obj}-\eqref{eq:ucb_mod_cmdp_c3} to obtain $\tilde{\pi}_k$. %by using the iterative algorithm of Section~\ref{appendix:primal_dual}.
			\item In case ~\eqref{eq:ucb_cmdp_obj}-\eqref{eq:ucb_cmdp_c3} is infeasible, choose $\tilde{\pi}_k$ to be some pre-determined policy (chosen at time $t=0$).
		\end{enumerate}
		\State \textbf{Implement} $\bm{\tilde{\pi}_k}$:
		\While{$t-\tau_{k_t}< \lceil T^{\alpha}\rceil $}
		\State{
			\begin{enumerate}
				\item Sample $a_t$ according to the distribution $\tilde{\pi}_k(\cdot|s_t)$. Observe reward $r(s_t,a_t)$, and observe next state $s_{t+1}$.
				\item Update $n_k(s_t,a_t)=n_k(s_t,a_t)+1$.
				\item Set $t:=t+1$.
			\end{enumerate}
		}
		\EndWhile
		\EndFor
	\end{algorithmic}
\end{algorithm}
The modified algorithm maintains empirical estimates $\hat{p}_t$ and confidence intervals $\cC_t$~\eqref{eq:ci_def} in exactly the same manner as UCRL-CMDP (Algorithm~\ref{algo:ucb_cmdp}) does. It also proceeds in episodes, and uses a single stationary control policy within an episode. However, at the beginning of each episode $k$, it solves the following optimization problem, which is a modification of the problem~\eqref{eq:ucb_cmdp_obj}-\eqref{eq:ucb_cmdp_c3} that is solved by UCRL-CMDP. More concretely, the cost constraints~\eqref{eq:ucb_cmdp_c2} are replaced by the constraints~\eqref{eq:ucb_mod_c2} on the costs:
\begin{align}
	&\max_{\mu,p^{\prime}} \sum_{(s,a)\in\cS\times\cA} \mu(s,a)r(s,a),\label{eq:ucb_mod_cmdp_obj}\\
	&\mbox{ s.t. } \sum_{(s,a)\in\cS\times\cA} \mu(s,a)c_i(s,a) \le c^{ub}_i-d_i, ~i\in [M] \label{eq:ucb_mod_c2}\\
	&\sum_{a\in\cA}\mu(s,a) = \sum_{(s^{\prime},b)}\mu(s^{\prime},b) p^{\prime}(s^{\prime},b,s),~~\forall s\in\cS,\label{eq:ucb_mod_cmdp_c1}\\
	&\mu(s,a)\ge 0~~\forall (s,a), \sum_{(s,a)} \mu(s,a) = 1,\\
	&~~\qquad p\up \in \cC_{\tau_{k}}, \label{eq:ucb_mod_cmdp_c3}
	%\eqref{eq:ucb_mod_cmdp_obj}-\eqref{eq:ucb_mod_cmdp_c3}
\end{align}  
where, 
\alig{
	d_i :&= b_i \eps,\quad  i\in [M],\label{def:d_i}
}
and the parameters $b_i \in (0,1),i\in [M]$ are chosen by the agent.~If the LP~\eqref{eq:ucb_mod_cmdp_obj}-\eqref{eq:ucb_mod_cmdp_c3} is feasible, let $\tilde{\mu}_k$ be an optimal occupation measure obtained by solving it. In this case, the agent implements $SR(\tilde{\mu}_k)$ within $\cE_k$. However, if the LP is infeasible, then it implements a stationary controller that has been chosen at time $t=0$. This is summarized in Algorithm~\ref{algo:ucb_cmdp_modi}. We will analyze Algorithm~\ref{algo:ucb_cmdp_modi} under the following assumption on the underlying MDP $p$. 

We derive upper-bounds on regrets of modified algorithm in the following result.
\begin{theorem}\label{th:regret_vector_result_1}
	Consider the modified UCRL-CMDP with $\delta = 1\slash T$, $\alpha = 1\slash 3$ applied to an MDP $p$ that satisfies Assumption~\ref{assum:1} and Assumption~\ref{assum:2}. Then, the expected reward and cost regrets can be upper-bounded as follows:%$(2\eps+z,\eps\bm{e}-\bm{d})$
\begin{align}
	& \bE \Delta\ur(T) \notag\\
	&\le 4 T_M\Bigg((\sqrt{2}+1)\sqrt{SAT}+ T^{2\slash 3} \sqrt{\log\Bigg( \frac{SAT}{\delta}\Bigg)}\Bigg)\notag \\
	&+\frac{C\lceil T^{2\slash 3}\rceil}{1-\rho}+\delta T + \frac{2}{T^{2b-2}|\cS||\cA|}+z T,
\end{align}
and
\begin{align}
	& \bE \Delta\uci(T) \notag\\
	&\le 4 T_M\Bigg((\sqrt{2}+1)\sqrt{SAT}+ T^{2\slash 3} \sqrt{\log\Bigg( \frac{SAT}{\delta}\Bigg)}\Bigg)\notag \\
	&+\frac{C\lceil T^{2\slash 3}\rceil}{1-\rho}+\delta T + \frac{2}{T^{2b-2}|\cS||\cA|}-b_i \eps T, i\in [M],
\end{align}
where $z = (\max_i b_i ) \frac{\hat{\eta}}{\eta}\eps$, $\eta$ is as in~\eqref{eq:eta_def} and 
$$
\hat{\eta}:= \max_{(s,a)\in\mathcal{S}\times\mathcal{A}}r(s,a) - \min_{(s,a)\in\mathcal{S}\times\mathcal{A}}r(s,a).
$$
\end{theorem}
\section{Proof of Theorem~\ref{th:regret_vector_result_1}}
Proof closely follows the proof of Theorem~\ref{regret:total}, hence we point out only the key differences. The modified UCRL algorithm assigns the following modified index to policy $\pi$, 
$$~\cI_k(\pi) :=  \max_{\te \in \cC_{\tau_k}} \Big\{ \bar{r}(\pi,\te): \bar{c}_i(\pi,\te) \le c^{ub}_i - d_i,~i\in [M]\Big\}.
$$
If for some $i$ we have $\bar{c}_i(\pi,\te) >c^{ub}_i - d_i,~\forall \te \in \cC_{\tau_{k}}$, then we set $\cI_k(\pi)=-\infty$.

The proof of next result is omitted since it is similar to that of Lemma~\ref{lemma:ub}.
\begin{lemma}\label{lemma:ub_1}
	Let $\pi$ be a stationary randomized policy. On the set $\cG$ we have that $\cI_k(\pi)=-\infty$ if $\bar{c}_i(\pi,p)>c^{ub}_i -d_i +\delta_{t}(\pi),~\mbox{for some } i\in [M]$. Also, $\cI_k(\pi)\le \bar{r}(\pi,p)+\delta_{t}(\pi)$.
\end{lemma}

The following result allows us to derive bounds on the instantaneous regrets.
\begin{lemma}\label{lemma:index_lb_1}
	If a stationary policy $\pi$ satisfies $\bar{c}_i(\pi,p) \le c^{ub}_i -d_i,~\forall i \in [M]$, then on $\cG$ its index satisfies $\cI_k(\pi) \ge\bar{r}(\pi,p)$. With $\pi$ set equal to the policy which solves the CMDP $\max_{\pi}  \bar{r}(\pi,p)$ such that $\bar{c}_i(\pi,p)\le c^{ub}_i - d_i,~\forall i\in[M]$,~on $\cG$ the index of such a policy satisfies $\cI_k(\pi) \ge r\ust - z$, where $z$ is as in Theorem~\ref{th:regret_vector_result_1}.
\end{lemma}
\begin{IEEEproof}
	We note that on the set $\cG$, the true MDP $p$ always belongs to $\cC_{\tau_{k}}$. Since $\bar{c}_i(\pi,p) \le c^{ub}_i - d_i,~\forall i\in[M]$ this means that the index of $\pi$ satisfies 
	\nal{
		\cI_k(\pi) & =  \max_{\te \in \cC_{\tau_k}} \Big\{ \bar{r}(\pi,\te): \bar{c}_i(\pi,\te) \le c^{ub}_i-d_i,~i\in [M]\Big\}\\
		&\ge \bar{r}(\pi,p).
	}
	It follows from Lemma~\ref{lemma:opt_bound} that the optimal value of the CMDP $\max_{\pi} \bar{r}(\pi,p)$, such that~$\bar{c}_i(\pi,p)\le c^{ub}_i-d_i,\forall i\in[M]$, is greater than or equal to $r\ust - z$. Hence, it follows from the discussion above that the index of the policy which is optimal for this CMDP is greater than or equal to $r\ust - z$.
\end{IEEEproof}
~As earlier, we bound the regret on the sets $\cG,\cG^{c}_1$ and $\cG^{c}_2$ separately. On $\cG$, the regret is bounded by the time spent playing sub-optimal policies. 
\begin{lemma}\label{lemma:learn_1}
	On the set $\cG$, for the modified UCRL-CMDP algorithm, the instantaneous reward regret can be bounded by $\delta_{t}(\pi_k)+z$, while the instantaneous cost regret associated with the $i$-th cost can be bounded by $\delta_{t}(\pi_k)-d_i$.
\end{lemma}
\begin{IEEEproof}
	It follows from Lemma~\ref{lemma:ub_1} that if $\bar{c}_i(\pi,p)>c^{ub}_i -d_i + \delta_{t}(\pi_k)$, then $\cI_k(\pi)=-\infty$. However, it is shown in Lemma~\ref{lemma:index_lb_1} that there is a policy $\tilde{\pi}$ whose index is greater than $r\ust-z$. Since index of $\pi$ is less than that of $\tilde{\pi}$, the policy $\pi$ will not be played by UCRL-CMDP. This means that $\bar{c}_i(\pi,p)\le c^{ub}_i -d_i + \delta_{t}(\pi_k)$, which shows that the instantaneous cost regret is bounded by $\delta_{t}(\pi_k)-d_i$.
	
	To bound the instantaneous reward regret, note that it was shown in Lemma~\ref{lemma:index_lb_1} that there is a policy with index greater than $r\ust-z$, and hence the index of $\pi_k$ is necessarily greater than $r\ust-z$.~Since from Lemma~\ref{lemma:ub_1} we have that the index of $\pi_k$ is upper-bounded by $\bar{r}(\pi,p)+\delta_{t}(\pi_k)$, we must have $\bar{r}(\pi,p)+\delta_{t}(\pi_k)\ge r\ust-z$, or $\bar{r}(\pi,p)\ge r\ust - \delta_{t}(\pi_k)-z$, which shows that the instantaneous reward regret is bounded by $\delta_{t}(\pi_k)+z$.
\end{IEEEproof}

\begin{IEEEproof}[Proof of Theorem~\ref{th:regret_vector_result_1}]
Since the proof closely follows that of Theorem~\ref{regret:total}, we only point out the key differences. The decomposition result~\ref{ineq:reg_dec} holds for reward as well cost regrets. Similarly, the regrets on $\cG^{c}_2$ and $\cG^{c}_1$ can be bounded by $\delta T$ and $\frac{2}{T^{2b-2}|\cS||\cA|}$ respectively. The only difference arises during bounding the terms $\sum_k \Delta\ur_k$ and $\sum_k \Delta\uci_k$. It follows from Lemma~\ref{lemma:learn_1} that the bound on $\sum_k \Delta\ur_k$ differs from~\eqref{ineq:4} by an additional term $z T$, and similarly that on $\sum_k \Delta\uci_k$ differs by $\eps b_i T$. The proof is then completed by summing the bounds on regrets over the sets $\cG,\cG^c_1,\cG^c_2$.
\end{IEEEproof}
\section{Achievable Regret Vectors}
Let $\bm{\lambda}\geq \bm{0}_{M}$. Consider the Lagrangian relaxation of~\eqref{eq:adapt_obj}-\eqref{eq:adapt_constr}, 
\begin{align}
	& \cL(\bm{\lambda};\pi) \notag\\
	&:= \liminf_{T\to\infty }\frac{1}{T} \mathbb{E}_{\pi}\sum_{t=1}^{T} \left\{r( s_t,a_t) +\bm{\lambda}\cdot \left( \bm{c^{ub}} - \bm{c}( s_t,a_t ) \right) \right\},
\end{align}
where $\bm{c}(s_t,a_t)$ is the vector that consists of costs $c_i(s_t,a_t),i\in[M]$.~Consider its associated dual function~\cite{bertsekas1997nonlinear}, $\mathcal{D}(\bm{\lambda}) := \max_{\pi} \mathcal{L}(\bm{\lambda};\pi)$, and the dual problem
\begin{align}
	\min_{\bm{\lambda} \ge \bm{0}}~\mathcal{D}(\bm{\lambda}).\label{ineq:dual_2}
\end{align}
Define the diameter $D(p)$ of MDP $p$ as follows, $D(p):= \max_{s,s\up}\min_{\pi} T^{\pi}_{s,s\up}$.~$D(p)$ is finite if $p$ is communicating~\cite{puterman2014markov}.
\begin{theorem}\label{th:lower_bound}
	Consider a learning algorithm $\phi$. Then, there is a problem instance such that
	the regrets $\Delta\ur(T), \{ \Delta\uci(T)\}_{i=1}^{M}$ under $\phi$ satisfy
	\begin{align}
		\bE_{\phi} \Delta\ur(T)  + \sum_{i=1}^{M} \lambda\ust_i \bE_{\phi} \Delta\uci(T)    \ge .015\cdot \sqrt{D(p)SAT},
	\end{align}
	where $\bm{\lambda\ust}$ is an optimal solution of the dual problem~\eqref{ineq:dual_2}.
\end{theorem}
\begin{IEEEproof} We begin by considering an auxiliary reward maximization problem that involves the same MDP $p$, but in which the reward received at time $t$ by the agent is equal to $r(s_t,a_t)+\bm{\lambda}\cdot \left( \bm{c^{ub}} - \bm{c}( s_t,a_t ) \right)$ instead of $r(s_t,a_t)$. However, there are no average cost constraints in the auxiliary problem. Let $\phi^{\prime}$ be a history dependent policy for this auxiliary problem. Denote its optimal reward by $r^{\star}(\bm{\lambda})$. Then, the regret for cumulative rewards collected by $\phi^{\prime}$ in the auxiliary problem is given by 
	\begin{align*}
		r^{\star}(\bm{\lambda})~T - \mathbb{E}_{\phi^{\prime}}\left[\sum_{t=1}^{T} r( s_t,a_t) + \bm{\lambda}\cdot \left( \bm{c^{ub}} - \bm{c}( s_t,a_t ) \right)\right].
	\end{align*}
	It follows from Theorem~5 of~\cite{jaksch2010near} that the controlled transition probabilities $p(s,a,s\up)$ of the underlying MDP can be chosen so that this regret is greater than $.015\cdot \sqrt{D(p)SAT}$, i.e.,
	\nal{ 
		&r^{\star}(\bm{\lambda})~T - \mathbb{E}_{\phi^{\prime}}\left[\sum_{t=1}^{T} r( s_t,a_t)  +\bm{\lambda}\cdot \left( \bm{c}( s_t,a_t )-\bm{c^{ub}} \right)\right] \\
		&\geq .015\cdot \sqrt{D(p)SAT}.
	}
	We observe that any valid learning algorithm for the constrained problem is also a valid algorithm for the auxiliary problem. Thus, if $\phi$ is a learning algorithm for the problem with average cost constraints, then we have
	\ali{
		& r^{\star}(\bm{\lambda})~T - \mathbb{E}_{\phi}\left[\sum_{t=1}^{T} r(s_t,a_t)  +\sum_{i=1}^{M} \lambda_i~\left(c^{ub}_i -  c_i( s_t,a_t ) \right)\right]\notag \\ 
		& \geq .015\cdot \sqrt{D(p)SAT}.\label{ineq:auer_regret}
	}
	We now substitute~\eqref{eq:lagrange} in the above to obtain
	\nal{
		&\bE_{\phi}~\Delta\ur(T) + \sum_{i=1}^{M}\lambda_i~\bE_{\phi}~\Delta\uci(T) \\
		&\geq .015\cdot \sqrt{D(p)SAT} + ~r^{\star} T -  r^{\star}(\bm{\lambda})~T.
	}
	Since the expression in the r.h.s. is maximized for values of $\bm{\lambda}$ which are optimal for the dual problem~\eqref{ineq:dual_2}, we set it equal to $\bm{\lambda^{\star}}$, and then use Lemma~\ref{lemma:slater_sd} in order to obtain  
	\begin{align}
		\bE_{\phi}~\Delta\ur(T) + \sum_{i=1}^{M}\lambda_i\bE_{\phi}~\Delta\uci(T) \geq .015\cdot \sqrt{D(p)SAT}.
	\end{align}
	This completes the proof. 
\end{IEEEproof}
\section{Simulation Results}
We compare the performance of the proposed UCRL-CMDP algorithm with the Actor-Critic algorithm for CMDPs that was proposed in~\cite{borkar2005actor}. Actor-Critic algorithms are a popular class of online learning algorithms~\cite{konda2000actor,peters2008natural,konda1999actor} that are based on multi-time-scale stochastic approximation~\cite{kushner2003stochastic,borkar2009stochastic}. We compare algorithms on the example presented in Section~\ref{sec:intro} in which the goal is to learn an efficient network controller. We begin by explaining the experiment setup.

\emph{Experiment Setup}: Consider the single-hop wireless network that was discussed in Section~\ref{sec:intro}, and consists of a single wireless node that transmits data packets to a receiver. The access point has to dynamically choose the transmission power $a_t$ at each time $t$. For simplicity, we let the action set $\cA$ be binary, and take the channel state to be static, i.e. it does not evolve or equivalently it assumes only a single value. Thus $a_t =0$ would mean that no packet was attempted transmission at time $t$, while $a_t =1$ would mean that a single packet would be delivered, with a probability equal to the channel reliability. The number of packets that arrive at time $t$ are denoted by $A_t$. We let $A_t \in \{0,1,2,3\}$ and assume that $A_t$ are i.i.d. across times. 
The probability vector associated with $A_t$ that describes its probability mass function is taken equal to $(.65,.2,.1,.05)$ for the experiments shown in Fig.~\ref{figure:1},~Fig.~\ref{figure:2}. The packet buffer is of a finite capacity, and can hold a maximum of $B$ packets. Thus, the dynamics of the queue length can be described as follows,
$$
Q_{t+1} = \left( Q_t + A_t - D_t  \right)^{+}   \wedge B,~t=0,1,2,\ldots,
$$
where for $x\in \bR$ we let $(x)^{+}:=\max\{x,0\}$, and $x\wedge B:= \min \{x,B\}$, while $D_t$ is the number that depart (are delivered to destination) at time $t$. In our experiments we use $B=6$, and take the channel reliability as $.9$. Hence, if $a_t=1$ then $D_t$ assumes the value $1$ with a probability $.9$. The associated CMDP can be stated as follows:
\ali{
	\max_{\pi}	\liminf_{T\to\infty }\frac{\bE_{\pi}\left( \sum\limits_{t=1}^{T}-a_t\right)}{T},\mbox{ s.t. } \limsup_{T\to\infty }\frac{\bE_{\pi}\left( \sum\limits_{t=1}^{T} Q_t\right)}{T} \le c^{ub}.\label{def:simu}
}

We now discuss the Actor-Critic algorithm for CMDPs. We begin with some notation that are required in order to discuss Actor-Critic algorithm. Let $\{a(n)\},\{b(n)\},\{c(n)\}$ be positive stepsize sequences satisfying $\sum_{n=1}^{\infty}a(n)=\infty,\sum_{n=1}^{\infty}b(n)=\infty,\sum_{n=1}^{\infty}c(n)=\infty$, $	\sum_{n=1}^{\infty}a^{2}(n)+\sum_{n=1}^{\infty}b^{2}(n)+\sum_{n=1}^{\infty}c^{2}(n)<\infty$,~and $\frac{b(n)}{a(n)} \to 0, \frac{c(n)}{b(n)} \to 0$.

In our experiments we use $a(n)=1\slash n$, $b(n)=1\slash (n\log n)$ and $c(n) = 1\slash (n\log^{2} n)$. Let $\mathcal{Q}:= \Big\{x\in \bR^{|\cA|-1}: x_i \ge 0~ \forall i, \sum_{j=1}^{|\cA|-1} x_j \le 1 \Big\}$ denote the simplex of subprobability vectors. Let $\Gamma(\cdot)$ denote the map that projects a vector onto $\mathcal{Q}$. Thus, if $x\in \mathcal{Q}$ then $\Gamma(x)=x$, otherwise $\Gamma(x)$ is the point from $\mathcal{Q}$ that is closest to $x$.

\textit{Actor-Critic Algorithm for CMDPs}: The algorithm carries out iterations for three quantities that evolve at different time-scales and are coupled. To begin with, it replaces the original constrained MDP by an unconstrained one by imposing a penalty upon constraint violation.  is held fixed, 
(or equivalently $r(s_t,a_t)-\tilde{\lambda}_t c(s_t,a_t) $, since the term $\tilde{\lambda}_t c^{ub}$ does not depend upon the controls) 
The instantaneous reward for this modified MDP is equal to $r(s_t,a_t)-\tilde{\lambda}_t \left(c(s_t,a_t) - c^{ub}\right)$ where  $\tilde{\lambda}_t\ge 0 $ is the price associated with the constraint violation. $\tilde{\lambda}_t\ge0$ is itself being tuned in an online way, though at a slower time-scale. $\tilde{\lambda}_t$ serves as an estimate of the optimal value of the dual variable for the original CMDP. In order to solve this unconstrained MDP, the algorithm keeps an estimate of the value function $V_t: \cS\mapsto \bR$, which is updated as follows,
\nal{
	& V_{t+1}(s) = V_t(s) + a(N_t(s))\id\{s_t = s\} \times \\
	& \Bigg[ r(s,u_t) +\tilde{\lambda}_t c(s,u_t) - V_t(s) - V_t(s\ust)+V_t(s_{t+1})\Bigg],
}
where $s\ust$ is a designated state. Let $\pi_t(a|s)$ denote the probability with which action $a$ is implemented in state $s$ at time $t$. Let $a\ust$ be a designated action. These probabilities are generated as follows. The algorithm maintains vectors $\hat{\pi}_{t}(s)=\left\{\hat{\pi}_{t}(a|s):a\in \cA\right\}$ for each state $s\in \cS$, and updates it as follows,
$$
\hat{\pi}_{t+1}(s) =  \Gamma\Big(  \hat{\pi}_{t}(s)  +  \star \Big), t=1,2,\ldots, 
$$
where,
\nal{
	& \star =\sum_{a\neq a\ust} b( N_t(s,a)) \times \id\left\{ s_t = s, a_t = a \right\} \hat{\pi}_{t}(s,a)\\
	& \times \Big[	V_t(s) + V_t(s\ust)	- r(s,a) + \tilde{\lambda}_t c(s,a)	-V_t(s_{t+1})	\Big]e_j,
}
where $e_{a}$ is the unit vector with a $1$ in the place corresponding to action $a$\footnote{We enumerate the available actions as $1,2,\ldots,|\cA|$.}. The probability for action $a\ust$ is computed as follows,
$$
\hat{\pi}_{t}(a\ust|s) = 1 - \sum_{a\neq a\ust} \hat{\pi}_{t}(a|s).
$$
The action probabilities $\pi_t$ are then generated from $\hat{\pi}_{t}$ as follows,
\nal{
	\pi_t(a|s) = (1-\epsilon_t)\hat{\pi}_t(a|s) +\frac{\epsilon_t}{|\cA|},~~a\in \cA,
}
where $\epsilon_t \to 0$. Finally, the price $\tilde{\lambda}_t$ is updated as follows,
$$
\tilde{\lambda}_{t+1} = \left[\tilde{\lambda}_{t} + \gamma_t\Big( c(s_t,a_t) - c^{ub} \Big) \right]^{+},
$$
where $c^{ub}$ is the threshold on average queue length as in~\eqref{def:simu}.

In our experiments we use $s\ust = B$, $a\ust = 0$ and $\epsilon_t = 1\slash t$.

\emph{Results}: Fig.~\ref{figure:1} compares the cumulative regrets incurred by these algorithms. We observe that the reward regret as well as cost regret of UCRL-CMDP are low. We observe a serious drawback of the Actor-Critic algorithm's performance, that the cost regret is prohibitively high. We then vary the budget $c^{ub}$ on the average queue length. These results are shown in Fig.~\ref{figure:2}. Once again, we make a similar observation, that UCRL-CMDP is effective in balancing both, the reward regret $\Delta^{(R)}(t)$ and the cost regret $\Delta^{(1)}(t)$, while the Actor-Critic algorithm yields a high cost regret. In both of these experiments the probability vector of arrivals was held fixed at $(.65,.2,.1,.05)$. We vary this probability vector, and plot the regrets in Fig.~\ref{fig:cost_arrival_vary}. Once again, UCRL-CMDP outperforms the Actor-Critic algorithm. Though the reward regret of Actor-Critic algorithm is lower than that of the UCRL-CMDP algorithms, this occurs at the expense of an undesireable much larger cost regret. In contrast, the reward regret as well as cost regret of UCRL-CMDP is low. Plots are obtained after averaging over $100$ runs.

\begin{figure}[H]
	\centering
	\subfloat[Reward Regret]{ 
		\includegraphics[width=4.5cm]{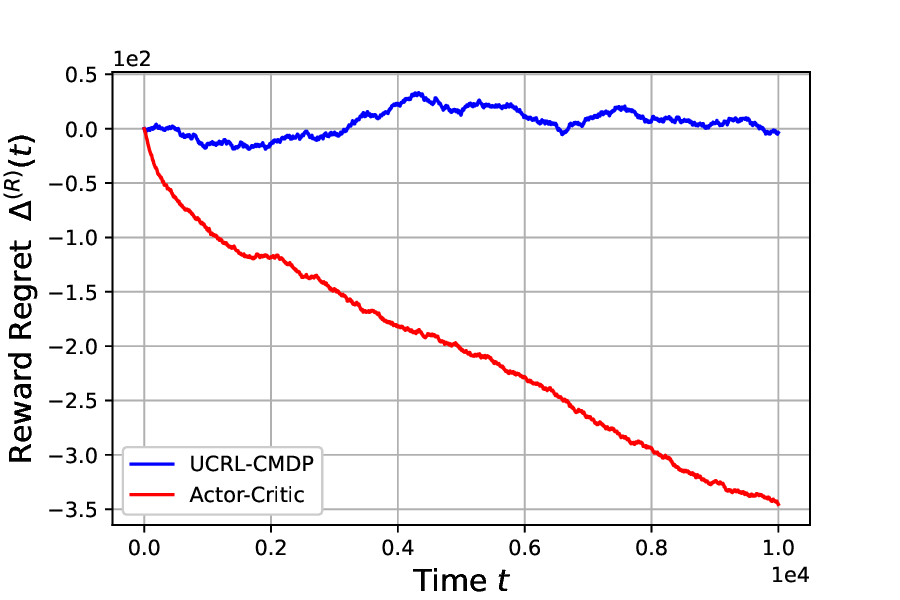}\label{fig:reward}
	}
	\subfloat[Cost Regret]{ 
		\includegraphics[width=4.5cm]{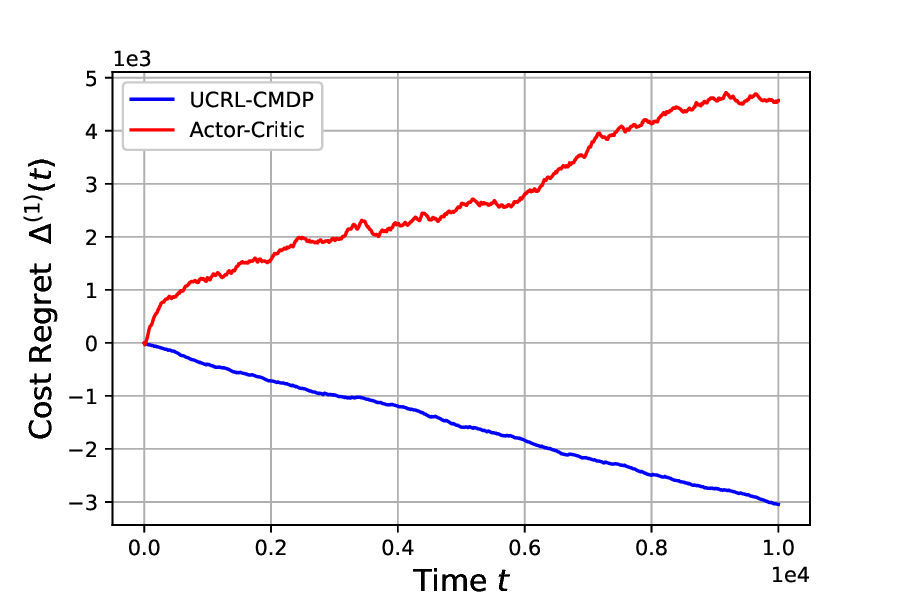}\label{fig:cost}
	}\\
	\caption{\begin{small}Plot of the reward regret (a) and cost regret (b), for the network in which the probability vector associated with arrivals is $(.65,.2,.1,.05)$, channel reliability is $.9$, and desired delay is $c^{ub} = 4.5$. Plots are obtained after averaging over $100$ runs.\end{small}}
	\label{figure:1}
	%	\label{s1}
\end{figure}

\begin{figure}[H]
	\centering
	\subfloat[Reward Regret]{ 
		\includegraphics[width=4.5cm]{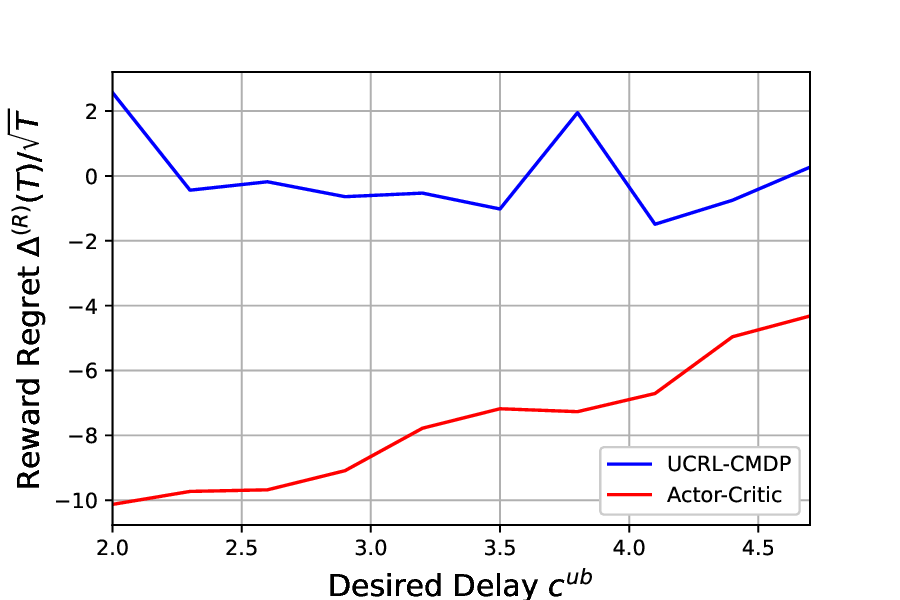}\label{fig:reward_des}
	}
	\subfloat[Cost Regret]{ 
		\includegraphics[width=4.5cm]{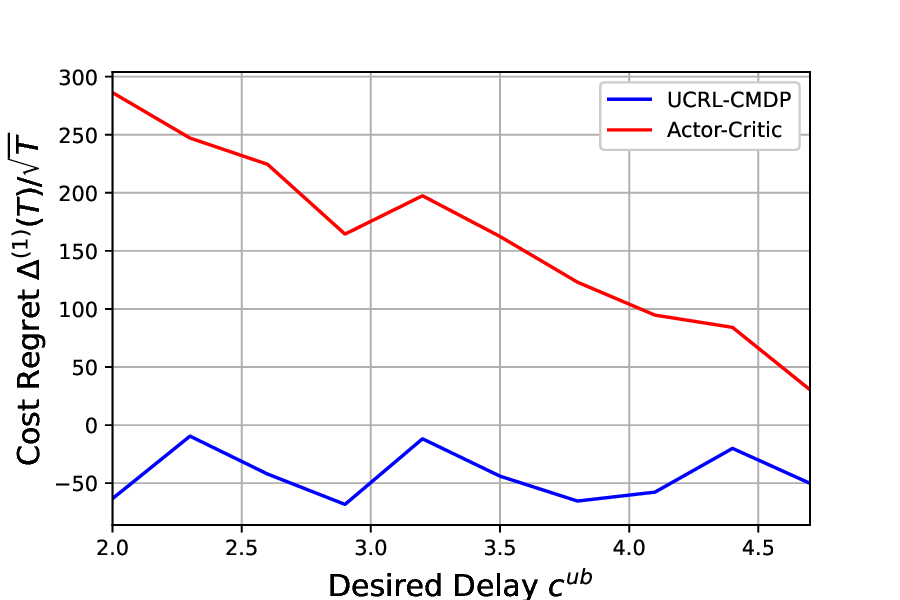}\label{fig:cost_des}
	}\\
	\caption{\begin{small}Plot of the normalized reward regret (a) and cost regret (b), as the desired delay $c^{ub}$ is varied. %Arrivals $A_t \in \{0,1,2,3\}$, with the probability vector equal to $(.65,.2,.1,.05)$, and channel reliability is held fixed at $.9$. Observe that in the second plot, the cost regret of Actor-Critic algorithm is prohibitively high, while UCRL-CMDP is able to keep the reward regret as well as cost regret low for all values of the desired bound on delay. Plots are obtained after averaging over $100$ runs.
	\end{small}}
	\label{figure:2}
	%	\label{s1}
\end{figure}

\begin{figure}[H]
	\centering
	\subfloat[Reward Regret]{ 
		\includegraphics[width=4.5cm]{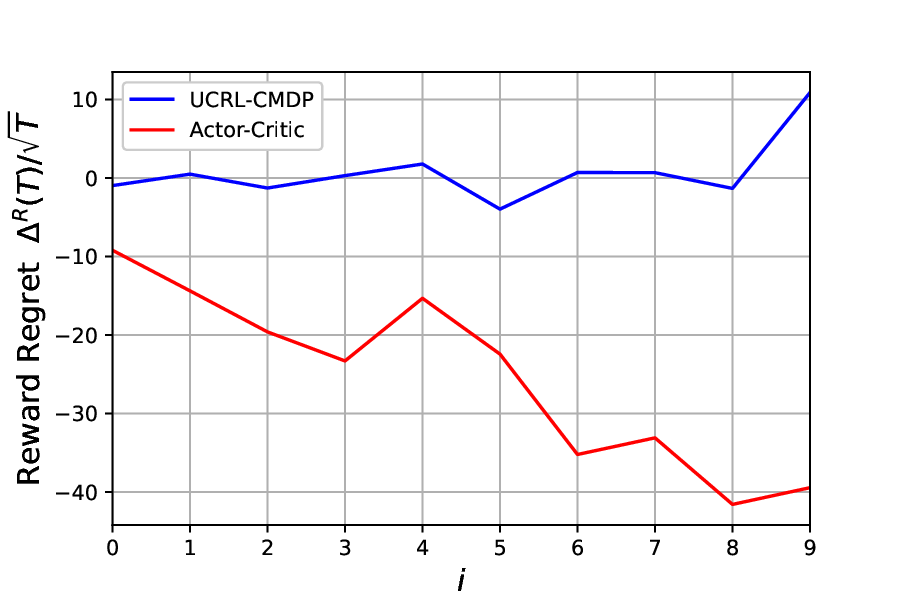}\label{fig:reward_arrival_vary}
	}
	\subfloat[Cost Regret]{ 
		\includegraphics[width=4.5cm]{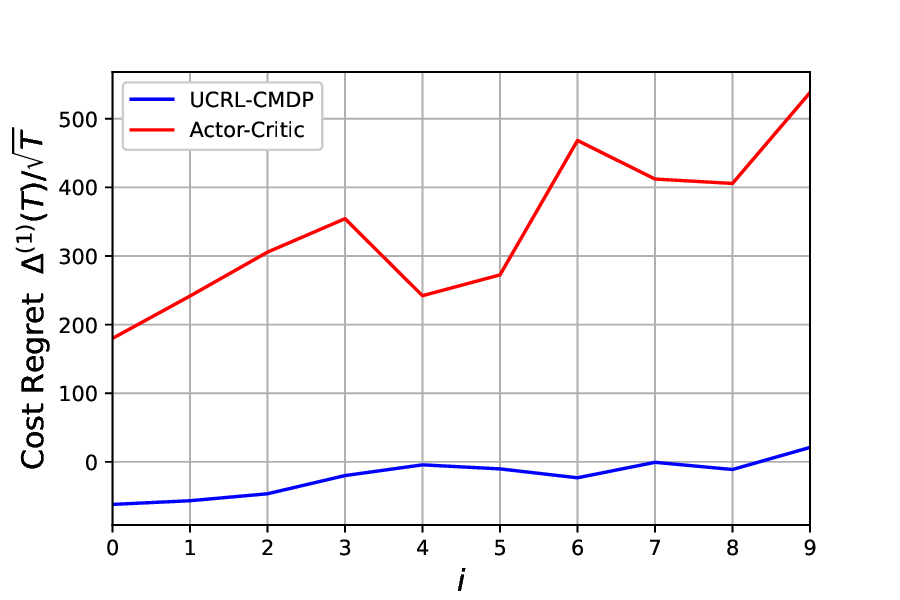}\label{fig:cost_arrival_vary}
	}\\
	\caption{\begin{small}Plot of the reward regret (a) and cost regret (b), as the probability distribution of the arrivals is varied. The probability vector of $A_t$ is equal to $(.65-.02i,.2,.1+.01i,.05+.01i)$, where the parameter $i$ is varied from $0$ to $9$. The desired delay $c^{ub}$ is held fixed at $4.5$, and channel reliability at $.9$.\end{small}}
	\label{figure:4}
\end{figure}

\section{Conclusions and Future Work}
In this work, we initiate a study to develop learning algorithms that simultaneously control all the components of the regret vector while controlling unknown MDPs. We devised algorithms that are able to tune different components of the cost regret vector, and also obtained a non-achievability result that characterizes those regret vectors that cannot be achieved under any learning rule.~In our work, we assume that the underlying MDP is unichain. An interesting research problem is to characterize the set of achievable regret vectors under the weaker assumption that the underlying MDP is communicating.

\bibliographystyle{IEEEtran}
\bibliography{references}

\begin{thebibliography}{10}
\providecommand{\url}[1]{#1}
\csname url@samestyle\endcsname
\providecommand{\newblock}{\relax}
\providecommand{\bibinfo}[2]{#2}
\providecommand{\BIBentrySTDinterwordspacing}{\spaceskip=0pt\relax}
\providecommand{\BIBentryALTinterwordstretchfactor}{4}
\providecommand{\BIBentryALTinterwordspacing}{\spaceskip=\fontdimen2\font plus
\BIBentryALTinterwordstretchfactor\fontdimen3\font minus
  \fontdimen4\font\relax}
\providecommand{\BIBforeignlanguage}[2]{{%
\expandafter\ifx\csname l@#1\endcsname\relax
\typeout{** WARNING: IEEEtran.bst: No hyphenation pattern has been}%
\typeout{** loaded for the language `#1'. Using the pattern for}%
\typeout{** the default language instead.}%
\else
\language=\csname l@#1\endcsname
\fi
#2}}
\providecommand{\BIBdecl}{\relax}
\BIBdecl

\bibitem{DBLP:books/lib/SuttonB98}
\BIBentryALTinterwordspacing
R.~S. Sutton and A.~G. Barto, \emph{Reinforcement learning - an introduction},
  ser. Adaptive computation and machine learning.\hskip 1em plus 0.5em minus
  0.4em\relax {MIT} Press, 1998. [Online]. Available:
  \url{http://www.worldcat.org/oclc/37293240}
\BIBentrySTDinterwordspacing

\bibitem{puterman2014markov}
M.~L. Puterman, \emph{{M}arkov Decision Processes.: Discrete Stochastic Dynamic
  Programming}.\hskip 1em plus 0.5em minus 0.4em\relax John Wiley \& Sons,
  2014.

\bibitem{sennott2009stochastic}
L.~I. Sennott, \emph{Stochastic dynamic programming and the control of queueing
  systems}.\hskip 1em plus 0.5em minus 0.4em\relax John Wiley \& Sons, 2009,
  vol. 504.

\bibitem{lazar1983optimal}
A.~Lazar, ``Optimal flow control of a class of queueing networks in
  equilibrium,'' \emph{IEEE transactions on Automatic Control}, vol.~28,
  no.~11, pp. 1001--1007, 1983.

\bibitem{hsiao1991optimal}
M.-T.~T. Hsiao and A.~A. Lazar, ``Optimal decentralized flow control of
  markovian queueing networks with multiple controllers,'' \emph{Performance
  evaluation}, vol.~13, no.~3, pp. 181--204, 1991.

\bibitem{nain1986optimal}
P.~Nain and K.~Ross, ``Optimal priority assignment with hard constraint,''
  \emph{IEEE transactions on Automatic Control}, vol.~31, no.~10, pp. 883--888,
  1986.

\bibitem{singh2018throughput}
R.~Singh and P.~Kumar, ``Throughput optimal decentralized scheduling of
  multihop networks with end-to-end deadline constraints: Unreliable links,''
  \emph{IEEE Transactions on Automatic Control}, vol.~64, no.~1, pp. 127--142,
  2018.

\bibitem{singh2021adaptive}
------, ``Adaptive csma for decentralized scheduling of multi-hop networks with
  end-to-end deadline constraints,'' \emph{IEEE/ACM Transactions on
  Networking}, 2021.

\bibitem{brafman2002r}
R.~I. Brafman and M.~Tennenholtz, ``R-max-a general polynomial time algorithm
  for near-optimal reinforcement learning,'' \emph{Journal of Machine Learning
  Research}, vol.~3, no. Oct, pp. 213--231, 2002.

\bibitem{auer2007logarithmic}
P.~Auer and R.~Ortner, ``Logarithmic online regret bounds for undiscounted
  reinforcement learning,'' in \emph{Advances in Neural Information Processing
  Systems}, 2007, pp. 49--56.

\bibitem{bartlett2009regal}
P.~L. Bartlett and A.~Tewari, ``{REGAL:} {A} regularization based algorithm for
  reinforcement learning in weakly communicating mdps,'' in \emph{Proceedings
  of the Twenty-Fifth Conference on Uncertainty in Artificial Intelligence,
  Montreal, QC, Canada, June 18-21, 2009}.\hskip 1em plus 0.5em minus
  0.4em\relax {AUAI} Press, 2009, pp. 35--42.

\bibitem{jaksch2010near}
T.~Jaksch, R.~Ortner, and P.~Auer, ``Near-optimal regret bounds for
  reinforcement learning,'' \emph{Journal of Machine Learning Research},
  vol.~11, no. Apr, pp. 1563--1600, 2010.

\bibitem{altman1991adaptive}
E.~Altman and A.~Schwartz, ``Adaptive control of constrained {Markov} chains,''
  \emph{IEEE Transactions on Automatic Control}, vol.~36, no.~4, pp. 454--462,
  1991.

\bibitem{borkar2005actor}
V.~S. Borkar, ``An actor-critic algorithm for constrained {M}arkov decision
  processes,'' \emph{Systems \& control letters}, vol.~54, no.~3, pp. 207--213,
  2005.

\bibitem{borkar1997stochastic}
------, ``Stochastic approximation with two time scales,'' \emph{Systems \&
  Control Letters}, vol.~29, no.~5, pp. 291--294, 1997.

\bibitem{achiam2017constrained}
J.~Achiam, D.~Held, A.~Tamar, and P.~Abbeel, ``Constrained policy
  optimization,'' \emph{arXiv preprint arXiv:1705.10528}, 2017.

\bibitem{liu2019ipo}
Y.~Liu, J.~Ding, and X.~Liu, ``Ipo: Interior-point policy optimization under
  constraints,'' \emph{arXiv preprint arXiv:1910.09615}, 2019.

\bibitem{tessler2018reward}
C.~Tessler, D.~J. Mankowitz, and S.~Mannor, ``Reward constrained policy
  optimization,'' \emph{arXiv preprint arXiv:1805.11074}, 2018.

\bibitem{uchibe2007constrained}
E.~Uchibe and K.~Doya, ``Constrained reinforcement learning from intrinsic and
  extrinsic rewards,'' in \emph{2007 IEEE 6th International Conference on
  Development and Learning}.\hskip 1em plus 0.5em minus 0.4em\relax IEEE, 2007,
  pp. 163--168.

\bibitem{qiu2020upper}
S.~Qiu, X.~Wei, Z.~Yang, J.~Ye, and Z.~Wang, ``Upper confidence primal-dual
  optimization: {S}tochastically constrained {M}arkov decision processes with
  adversarial losses and unknown transitions,'' \emph{arXiv preprint
  arXiv:2003.00660}, 2020.

\bibitem{efroni2020exploration}
Y.~Efroni, S.~Mannor, and M.~Pirotta, ``{Exploration-Exploitation in
  Constrained MDPs},'' \emph{arXiv preprint arXiv:2003.02189}, 2020.

\bibitem{villani2008optimal}
C.~Villani, \emph{Optimal transport: old and new}.\hskip 1em plus 0.5em minus
  0.4em\relax Springer Science \& Business Media, 2008, vol. 338.

\bibitem{resnick2019probability}
S.~Resnick, \emph{A probability path}.\hskip 1em plus 0.5em minus 0.4em\relax
  Springer, 2019.

\bibitem{kumar2015stochastic}
P.~R. Kumar and P.~Varaiya, \emph{Stochastic systems: Estimation,
  identification, and adaptive control}.\hskip 1em plus 0.5em minus 0.4em\relax
  SIAM, 2015.

\bibitem{altman}
E.~Altman, \emph{Constrained {M}arkov Decision Processes}.\hskip 1em plus 0.5em
  minus 0.4em\relax Chapman and Hall/CRC, March 1999.

\bibitem{lai1985asymptotically}
T.~L. Lai and H.~Robbins, ``Asymptotically efficient adaptive allocation
  rules,'' \emph{Advances in applied mathematics}, vol.~6, no.~1, pp. 4--22,
  1985.

\bibitem{agrawal1995sample}
R.~Agrawal, ``Sample mean based index policies by o (log n) regret for the
  multi-armed bandit problem,'' \emph{Advances in Applied Probability},
  vol.~27, no.~4, pp. 1054--1078, 1995.

\bibitem{azuma1967weighted}
K.~Azuma, ``Weighted sums of certain dependent random variables,'' \emph{Tohoku
  Mathematical Journal, Second Series}, vol.~19, no.~3, pp. 357--367, 1967.

\bibitem{bertsekas1997nonlinear}
D.~P. Bertsekas, \emph{Nonlinear programming}.\hskip 1em plus 0.5em minus
  0.4em\relax Taylor \& Francis, 1997, vol.~48, no.~3.

\bibitem{konda2000actor}
V.~R. Konda and J.~N. Tsitsiklis, ``Actor-critic algorithms,'' in
  \emph{Advances in neural information processing systems}, 2000, pp.
  1008--1014.

\bibitem{peters2008natural}
J.~Peters and S.~Schaal, ``Natural actor-critic,'' \emph{Neurocomputing},
  vol.~71, no. 7-9, pp. 1180--1190, 2008.

\bibitem{konda1999actor}
V.~R. Konda and V.~S. Borkar, ``Actor-critic--type learning algorithms for
  markov decision processes,'' \emph{SIAM Journal on control and Optimization},
  vol.~38, no.~1, pp. 94--123, 1999.

\bibitem{kushner2003stochastic}
H.~Kushner and G.~G. Yin, \emph{Stochastic approximation and recursive
  algorithms and applications}.\hskip 1em plus 0.5em minus 0.4em\relax Springer
  Science \& Business Media, 2003, vol.~35.

\bibitem{borkar2009stochastic}
V.~S. Borkar, \emph{Stochastic approximation: a dynamical systems
  viewpoint}.\hskip 1em plus 0.5em minus 0.4em\relax Springer, 2009, vol.~48.

\bibitem{boyd2004convex}
S.~Boyd and L.~Vandenberghe, \emph{Convex optimization}.\hskip 1em plus 0.5em
  minus 0.4em\relax Cambridge university press, 2004.

\bibitem{mitrophanov2005sensitivity}
A.~Y. Mitrophanov, ``Sensitivity and convergence of uniformly ergodic markov
  chains,'' \emph{Journal of Applied Probability}, vol.~42, no.~4, pp.
  1003--1014, 2005.

\end{thebibliography}
\appendices 
\section{Results Used in the Proof of Theorem~\ref{th:lower_bound}}
We derive some preliminary results that will be utilized in the proof of Theorem~\ref{th:lower_bound}.
\begin{lemma}\label{lemma:slater_sd}
	Consider the dual problem~\eqref{ineq:dual_2} associated with the CMDP~\eqref{eq:adapt_obj},~\eqref{eq:adapt_constr}, and let $\bm{\lambda^{\star}}$ be a solution of the dual problem. If Assumption~\ref{assum:2} holds true, then we have that
	\begin{align}
		\mathcal{D}(\bm{\lambda^{\star}}) = r^{\star},
	\end{align} 
	where $r\ust$ is the optimal reward of CMDP~\eqref{eq:adapt_obj},~\eqref{eq:adapt_constr}.
\end{lemma}
\begin{IEEEproof}
	Under Assumption~\ref{assum:2}, the CMDP~\eqref{eq:adapt_obj}-\eqref{eq:adapt_constr} is strictly feasible, so that Slater's constraint~\cite{boyd2004convex} is satisfied, and consequently strong duality holds true. Thus, if $\lambda^{\star}$ solves the dual problem~\eqref{ineq:dual_2}, we then have that $\mathcal{D}(\bm{\lambda^{\star}}) = r^{\star}$.
\end{IEEEproof}

\begin{lemma}
	Let $\bm{\lambda}\ge \bm{0}_M$ and $\phi$ be a learning algorithm for the problem of maximizing cumulative rewards under average cost constraints. We then have the following,
	\begin{align}\label{eq:lagrange}
		& \bE_{\phi}\sum_{t=1}^{T} \left\{r( s_t,a_t) +\bm{\lambda}\cdot (\bm{c^{ub}} - \bm{c}(s_t,a_t )) \right\}\notag \\
		&\qquad = r^{\star}T  -\bE_{\phi} \Delta\ur(T) - \sum_{i=1}^{M}\lambda_i\bE_{\phi} ~\Delta\uci(T).
	\end{align}
\end{lemma}
\begin{IEEEproof} We have,
	\begin{align*}
		&\bE_{\phi}\sum_{t=1}^{T} \left\{r( s_t,a_t) + \bm{\lambda}\cdot (\bm{c^{ub}} - \bm{c}(s_t,a_t ))\right\}\\
		&=  \bE_{\phi}\sum_{t=1}^{T} r( s_t,a_t) + \sum_{i=1}^{M} \lambda_i~\mathbb{E}_{\phi}\sum_{t=1}^{T} \left( c^{ub}_i - c_i( s_t,a_t ) \right)\notag\\
		&= r^{\star}T - \left(r^{\star}T - \bE_{\phi}\sum_{t=1}^{T} r( s_t,a_t)  \right)\\
		& - \sum_{i=1}^{M} \lambda_i~\bE_{\phi}\sum_{t=1}^{T} \left( c_i( s_t,a_t ) -  c^{ub}_i\right)\notag\\
		&= r^{\star}T  -\bE_{\phi} \Delta\ur(T) - \sum_{i=1}^{M}\lambda_i\bE_{\phi} ~\Delta\uci(T).
	\end{align*} 
\end{IEEEproof}
\section{Some Auxiliary Results}
\subsection{Perturbation Analysis of CMDPs}
We derive some results on the variations in the value of optimal reward of the CMDP~\eqref{eq:adapt_obj}-\eqref{eq:adapt_constr} as a function of the cost budgets $\bm{c^{ub}}$. Consider a vector $\bm{\hat{c}^{ub}}$ of cost budgets that satisfies
\begin{align}\label{ineq:c_range}
	c^{ub}_i-\epsilon \le \hat{c}^{ub}_i \le c^{ub}_i,~\forall i\in [M],
\end{align}
where $\epsilon>0$. Now consider the following CMDP in which the upper-bounds on the average costs are equal to $\{\hat{c}^{ub}_i\}_{i=1}^{M}$. 
\begin{align}
	\max_{\pi} &\liminf_{T\to\infty }\frac{1}{T} \mathbb{E}_{\pi}\sum_{t=1}^{T} r( s_t,a_t) \label{eq:cmdp_modf_obj}\\
	\mbox{ s.t. }&  \limsup_{T\to\infty }\frac{1}{T} \mathbb{E}_{\pi}\sum_{t=1}^{T} c_i( s_t,a_t ) \leq \hat{c}^{ub}_i\label{eq:cmdp_modf_constr}, i\in [1,M].
\end{align}
\begin{lemma}\label{lemma:lagrange_ub} 
	Let the MDP $p$ satisfy Assumption~\ref{assum:1} and Assumption~\ref{assum:2}. Let $\bm{\lambda}\ust$ be an optimal dual variable\slash Lagrange multiplier associated with the CMDP~\eqref{eq:cmdp_modf_obj}-\eqref{eq:cmdp_modf_constr}. Then, $\bm{\lambda}\ust$ satisfies $\sum_{i=1}^{M} \lambda\ust_i \le \frac{\hat{\eta} }{\eta}$,
	where the constant $\eta$ is as in~\eqref{eq:eta_def}, while $\hat{\eta}$ is as in Theorem~\ref{regret_vector_result_1}.
\end{lemma}
\begin{IEEEproof} Within this proof, we let $\pi^{\star}(\hat{c}^{ub})$ denote an optimal stationary policy for~\eqref{eq:cmdp_modf_obj}-\eqref{eq:cmdp_modf_constr}. Recall that the policy $\pi_{feas.}$ that was defined in Assumption~\ref{assum:2} satisfies $\bar{c}_i(\pi_{feas.})\le c^{ub}_i-\eta$.
	We have
	\begin{align*}
		&\max_{(s,a)\in\mathcal{S}\times\mathcal{A}}r(s,a) \geq  \bar{r}(\pi^{\star}(\hat{c}^{ub}))\\
		&= \bar{r}(\pi^{\star}(\hat{c}^{ub})) + \sum_{i=1}^{M} \lambda\ust_i \left(\hat{c}^{ub}_i-\bar{c}_i( \pi^{\star}(\hat{c}^{ub})  \right) \\
		&\geq\bar{r}(\pi_{feas.}) + \sum_{i=1}^{M}\lambda\ust_i \left(\hat{c}^{ub}-\bar{c}(\pi_{feas.}) \right) \\
		&\geq \min_{(s,a)\in \mathcal{S}\times\mathcal{A} }r(s,a) +  \sum_{i=1}^{M}\lambda\ust_i \left(\hat{c}^{ub}-\bar{c}(\pi_{feas.}) \right) \\
		&\geq \min_{(s,a)\in \mathcal{S}\times\mathcal{A} }r(s,a) + \eta \sum_{i=1}^{M} \lambda\ust_i,
	\end{align*}
	where the second inequality follows since a policy that is optimal for the problem~\eqref{eq:cmdp_modf_obj}-\eqref{eq:cmdp_modf_constr} maximizes the Lagrangian $\bar{r}(\pi) + \sum_{i=1}^{M} \lambda_i \left(\hat{c}^{ub}_i-\bar{c}_i( \pi)  \right)$ when the Lagrange multiplier $\bm{\lambda}$ is set equal to  $\bm{\lambda\ust}$~\cite{bertsekas1997nonlinear}. Rearranging the above inequality yields the desired result. 
\end{IEEEproof}

\begin{lemma}\label{lemma:opt_bound}
	Let the MDP $p$ satisfy Assumption~\ref{assum:1} and Assumption~\ref{assum:2}. If $r^{\star}(\bm{\hat{c}^{ub}})$ denotes optimal reward value of~\eqref{eq:cmdp_modf_obj}-\eqref{eq:cmdp_modf_constr}, and $r\ust$ is optimal reward of problem~\eqref{eq:adapt_obj}-\eqref{eq:adapt_constr}, then we have that
	\begin{align*}
		r^{\star} - r^{\star} (\bm{\hat{c}^{ub}}) \leq \left(\max_{i\in[1,M]} \left\{c^{ub}_i - \hat{c}^{ub}_i \right\} \right) \frac{\hat{\eta}}{\eta},
	\end{align*}
	where $\hat{\eta}$ is as in Theorem~\ref{regret_vector_result_1}, $\eta$ is as in~\eqref{eq:eta_def}, and $\bm{\hat{c}}$ satisfies~\eqref{ineq:c_range}.
\end{lemma}
\begin{IEEEproof} 
	As discussed in Section~\ref{subsec:cmdp}, a CMDP can be posed as a linear program.
	Since under Assumption~\ref{assum:2}, both the CMDPs~\eqref{eq:adapt_obj}-\eqref{eq:adapt_constr} and~\eqref{eq:cmdp_modf_obj}-\eqref{eq:cmdp_modf_constr} are strictly feasible, we can use the strong duality property of linear programs~\cite{bertsekas1997nonlinear} in order to conclude that the optimal value of the primal and the dual problems for both the CMDPs are equal. Thus,
	\begin{align}
		r^{\star} &= \sup_{\pi} \inf_{\lambda} ~~\bar{r}(\pi)+\sum_{i=1}^{M} \lambda_i\left(c^{ub}_i - \bar{c}_i(\pi)\right),\label{ineq:12}\\
		r^{\star} (\bm{\hat{c}^{ub}}) &= \sup_{\pi} \inf_{\lambda}~~ \bar{r}(\pi)+\sum_{i=1}^{M} \lambda_i \left(\hat{c}_i^{ub} - \bar{c}_i(\pi)\right).\label{ineq:13}
	\end{align}
	Let $\pi^{(1)},\pi^{(2)}$ and $\lambda^{(1)},\lambda^{(2)}$ denote optimal policies and vector consisting of optimal dual variables for the two CMDPs. It then follows from~\eqref{ineq:12} and~\eqref{ineq:13} that, 
	\begin{align*}
		r^{\star} &\leq  \bar{r}(\pi^{(1)})+\sum_{i=1}^{M}\lambda^{(2)}_i  \left( c^{ub}_i- \bar{c}_i(\pi^{(1)}) \right),\\
		\mbox{ and }~~r^{\star} (\bm{ \hat{c}^{ub}}) &\geq \bar{r}(\pi^{(1)})+ \sum_{i=1}^{M} \lambda^{(2)}_i \left(\hat{c}^{ub}_i- \bar{c}_i(\pi^{(1)}) \right).
	\end{align*}
	Subtracting the second inequality from the first yields
	\begin{align*}
		r^{\star} - r^{\star} (\bm{c^{ub}}) &\leq \sum_{i=1}^{M} \lambda^{(2)}_i\left( c^{ub}_i - \hat{c}^{ub}_i \right)\\
		&\leq \left(\max_{i\in [1,M]} \left\{c^{ub}_{i} - \hat{c}^{ub}_i \right\} \right) \left( \sum_{i=1}^{M}\lambda^{(2)}_i\right)\\
		&\leq  \left(\max_{i\in [1,M]} \left\{c^{ub}_{i} - \hat{c}^{ub}_i \right\} \right) \frac{\hat{\eta}}{\eta},
	\end{align*}
	where the last inequality follows from Lemma~\ref{lemma:lagrange_ub}. This completes the proof. 
\end{IEEEproof}
\subsection{Sensitivity of Markov Chains}
The following result is essentially Corollary 3.1 of~\cite{mitrophanov2005sensitivity}. Consider a finite-state Markov chain with transition probabilities $\{\tilde{p}(s,s\up):s,s\up\in\cS\}$. Let $P^{(t)}_{s}$ be the probability distribution at time $t$ when it starts in state $s$ at time $0$.
\begin{theorem}\label{th:mitro}
	Assume $\|\tilde{P}^{(t)}_{s} - \tilde{P}^{(\infty)}_{s}\|\le C\rho^{t},~t\in\bN$. Consider a Markov chain with transition probabilities $\tilde{q}(s,s\up)$. We have 
	$$
	\|\tilde{P}^{(\infty)} - \tilde{Q}^{(\infty)} \| \le \left( \hat{n} + \frac{C\rho^{\hat{n}}}{1-\rho}  \right) \|\tilde{p}-\tilde{q}\|,
	$$
	where $\hat{n}:= \lceil \log_{\rho} C^{-1}\rceil$.
\end{theorem}

\end{document}